\documentclass{article}

\PassOptionsToPackage{numbers, compress}{natbib}

\usepackage[numbers]{natbib}
\usepackage{url}
\usepackage{subcaption}  
\usepackage[letterpaper,top=2cm,bottom=2cm,left=3cm,right=3cm,marginparwidth=1.75cm]{geometry}
\usepackage{palatino}
\usepackage{jmlrutils}

\usepackage[utf8]{inputenc} %
\usepackage[T1]{fontenc}    %
\usepackage{hyperref}       %
\usepackage{url}            %
\usepackage{booktabs}       %
\usepackage{amsfonts}       %
\usepackage{nicefrac}       %

\usepackage{times}
\usepackage[utf8]{inputenc}
\usepackage[T1]{fontenc}
\usepackage{booktabs}
\usepackage{amsfonts}
\usepackage[noend]{algorithmic}
\usepackage{nicefrac}
\usepackage{soul}
\usepackage{backref}

 \renewcommand*{\backrefalt}[4]{%
    \ifcase #1%
     \or Cited on page:~#2%
     \else Cited on pages:~#2%
    \fi%
    }

\renewcommand{\proofname}{Proof.}
\allowdisplaybreaks[4]

\usepackage{bm}

\usepackage{dsfont}
\usepackage{mathtools}
\usepackage{thmtools}
\usepackage{thm-restate}
\usepackage{tikz}
\usepackage{hyperref}

\usepackage{booktabs}       %
\usepackage{amsfonts}       %
\usepackage{nicefrac}       %
\usepackage{bbding}
\usepackage{multicol}
\usepackage{multirow}
\usepackage{courier}
\usepackage{mathrsfs}
\usepackage{setspace}
\usepackage{tikz}

\usepackage{makecell}
\usepackage{bm}
\usepackage{appendix}
\usepackage{comment}
\usepackage{tablefootnote}
\usepackage{multirow}
\usepackage{hhline}%
\usepackage[font=footnotesize,labelfont=bf]{caption}
\usepackage{cases}

\usepackage{enumitem}

\usepackage{algorithm}

\title{Conformal Prediction Meets Long-tail Classification}

\usepackage{times}

\author{
Shuqi Liu, Jianguo Huang, and Luke Ong\\
~\\
{\centering Nanyang Technological University}
}
\begin{document}
\date{}
\maketitle

\begin{abstract}%
Conformal Prediction (CP) is a popular method for uncertainty quantification that converts a pretrained model's point prediction into a prediction set, with the set size reflecting the model's confidence. Although existing CP methods are guaranteed to achieve marginal coverage, they often exhibit imbalanced coverage across classes under long-tail label distributions, tending to over cover the head classes at the expense of under covering the remaining tail classes. This under coverage is particularly concerning, as it undermines the reliability of the prediction sets for minority classes, even with coverage ensured on average. In this paper, we propose the Tail-Aware Conformal Prediction (TACP) method to mitigate the under coverage of the tail classes by utilizing the long-tail structure and narrowing the head-tail coverage gap. Theoretical analysis shows that it consistently achieves a smaller head-tail coverage gap than standard methods. To further improve coverage balance across all classes, we introduce an extension of TACP: soft TACP (sTACP) via a reweighting mechanism. The proposed framework can be combined with various non-conformity scores, and experiments on multiple long-tail benchmark datasets demonstrate the effectiveness of our methods. 

\end{abstract}

\section{Introduction}
Uncertainty quantification (UQ) is crucial in building reliable machine learning systems, especially in high-stakes applications, such as autonomous driving \cite{grigorescu2020survey}, where overconfident errors can be costly. 
A popular approach for UQ is Conformal Prediction (CP) \cite{vovk2005algorithmic}, which converts the point predictions of a pretrained model into prediction sets that come with rigorous statistical guarantees. 
Specifically, CP ensures that the prediction set contains the true label with a user-specified probability on average, and the size of the prediction set serves as a proxy for model confidence: smaller sets imply higher confidence.

Despite their strong theoretical guarantees, standard CP methods often struggle under long-tail (LT) label distributions \cite{buda2018systematic} leading to imbalanced coverage across classes.
A long-tail distribution refers to a class frequency distribution in which a small subset of classes accounts for the majority of samples, whereas the remaining classes contain comparatively few samples. The high-frequency classes are commonly referred to as head classes, while the low-frequency classes are termed tail classes.
CP methods typically achieve the desired marginal coverage on average, but display specific coverage imbalance: over covering the head classes at the cost of under covering the tail classes \cite{lofstrom2015bias, lu2022fair}. Conformal prediction sets often fail to include the ground-truth tail label with the target probability, resulting in unreliable uncertainty estimations. Such under coverage can mislead subsequent downstream decision-making, potentially causing costly or even harmful outcomes in fairness-sensitive or safety-critical scenarios.

To address this coverage imbalance observed in standard CP, group-conditional and class-conditional approaches have been proposed \cite{DBLP:conf/nips/DingABJT23,ShiGBD024,DBLP:journals/ml/Vovk13}, which aim to achieve coverage guarantee for each group or individual class.
However, these methods are often less practical when the calibration data is limited and exhibits a long-tail distribution, as the per-group or per-class acceptance thresholds estimates become highly unreliable and lead to excessively large prediction sets.
These limitations highlight the need for conformal prediction methods that provide valid, efficient, and balanced uncertainty estimates between groups with head and tail classes or further across classes in LT settings.

To tackle the problems above, we propose the \textbf{T}ail-\textbf{A}ware \textbf{C}onformal \textbf{P}rediction (TACP) method, which introduces a tailored regularization term that adapts to the underlying label imbalance, leading to more balanced coverage between head and tail classes. Then we provide a theoretical analysis demonstrating that TACP narrows the coverage gap between head and tail classes compared with standard CP.
To further improve coverage balance across all individual classes, we introduced an extension of TACP: \textbf{s}oft \textbf{T}ail-\textbf{A}ware \textbf{C}onformal \textbf{P}rediction (sTACP), which employs a smooth reweighting strategy for adaptive penalty control.
Our proposed frameworks are flexible and can be combined with a wide range of representative non-conformity scores.

Extensive experiments on multiple long-tail benchmarks, including CIFAR100-LT and ImageNet-LT, demonstrate the effectiveness of the proposed TACP and sTACP frameworks. TACP significantly reduces the head-tail coverage gap compared to standard CP across four non-conformity scores, while maintaining informative prediction sets. For example, on ImageNet-LT ($\rho = 0.6$) using the APS score, TACP reduces CovGap-HT from $2.18$ to $1.11$, accompanied by a decrease in AvgSize (36.43 $\rightarrow$ 33.98). sTACP further improves class-conditional coverage balance, outperforming other baselines on ImageNet-LT. For instance, with APS as the base score at $\alpha = 10\%$, sTACP reduces the class-conditional gap from $19.00\%$ (STANDARD) to $15.86\%$ while preserving a similar set size. These results highlight that TACP and sTACP provide robust, efficient, and balance uncertainty estimates under long-tail settings.

We summarize our contributions as follows:
\begin{itemize}
   \item We study the problem of CP under LT label distributions and propose the TACP framework to mitigate the coverage gap between head and tail classes.
   \item We provide a theoretical analysis to demonstrate the effectiveness of TACP and empirically compare its performance with several baselines across multiple long-tail benchmarks using four non-conformity scores.
   \item To narrow class-conditional coverage gaps in LT scenarios, we propose sTACP, an enhanced variant of TACP. Extensive experiments demonstrate that sTACP consistently reduces the class-conditional coverage gap across a wide range of long-tail classification tasks.
\end{itemize}

\section{Related Work}
Conformal prediction is an active area of research that provides a model-free and distribution-free framework for UQ by converting the point prediction of any model into a prediction set that contains the ground-truth label with a user-defined level. A widely used variant of CP is split CP \cite{lei2018distribution, papadopoulos2002inductive}, which employs a held-out calibration set to improve the computational efficiency of full CP \cite{vovk2005algorithmic}. In addition, several variants of CP have been developed based on cross-validation \cite{vovk2015cross} or leave-one-out \cite{barber2021predictive,lee2025leaveoneout}. 
There is an increasing amount of research focused on maintaining marginal coverage by relaxing the exchangeability assumption to account for
covariate shift \cite{tibshirani2019conformal} and label shift \cite{podkopaev2021distribution,xu2025wasserstein}.

CP has been applied to various tasks, including classification \citep{RAPS2021,lei2013distribution,DBLP:conf/nips/RomanoSC20,zhou2024conformal}, regression \citep{romano2019conformalized,seedat2023improving,sesia2021conformal}, and outlier detection \cite{bates2023testing,guan2022prediction}. 
Its performance is typically assessed by coverage validity and length efficiency. Hence, a central research goal is to improve length efficiency
\citep{DBLP:conf/icml/HuangXZYQW24, liu2024c,xi2025does}.
Another important line of research seeks to move beyond marginal coverage guarantees, including methods that aim to improve class-conditional coverage \citep{plassier2024probabilistic,sadinle2019least,shi2013applications}, feature-conditional coverage \citep{DBLP:conf/nips/RomanoSC20}.
Such conditional guarantees are especially important in LT data settings, where marginal coverage can obscure significant disparities across different classes or subgroups \citep{Kevin2023,lofstrom2015bias}.

Label distributions in real-world applications are often long-tail \cite{buda2018systematic}, where a few head classes dominate the data while many tail classes have a limited number of samples. Under such LT settings, previous works focused on improving models' prediction accuracy by 
data re-sampling \cite{brodersen2010balanced,chawla2002smote}
and surrogate loss function modification \cite{cao2019learning,tan2020equalization}.

Previous work has observed the phenomenon of imbalanced coverage when applying conformal prediction to long-tail or imbalanced label distributions  \cite{lofstrom2015bias,lu2022fair}.
\citet{Kevin2023} empirically showed that performance degrades in LT regimes and highlighted the coverage imbalance across classes.
Moreover, \citet{ShiGBD024} considered a partial LT setting, where the training dataset is highly imbalanced, while the calibration and test datasets used for CP remain balanced.
In contrast, we consider a fully long-tail setting, where the training, calibration, and test datasets all follow long-tail label distributions. This setting poses challenges for achieving reliable coverage guarantees across classes, and CP methods tailored for this scenario remain relatively undeveloped.

\begin{figure}
    \centering
    \begin{subfigure}[t]{0.48\textwidth}
        \centering
        \includegraphics[width=\linewidth]{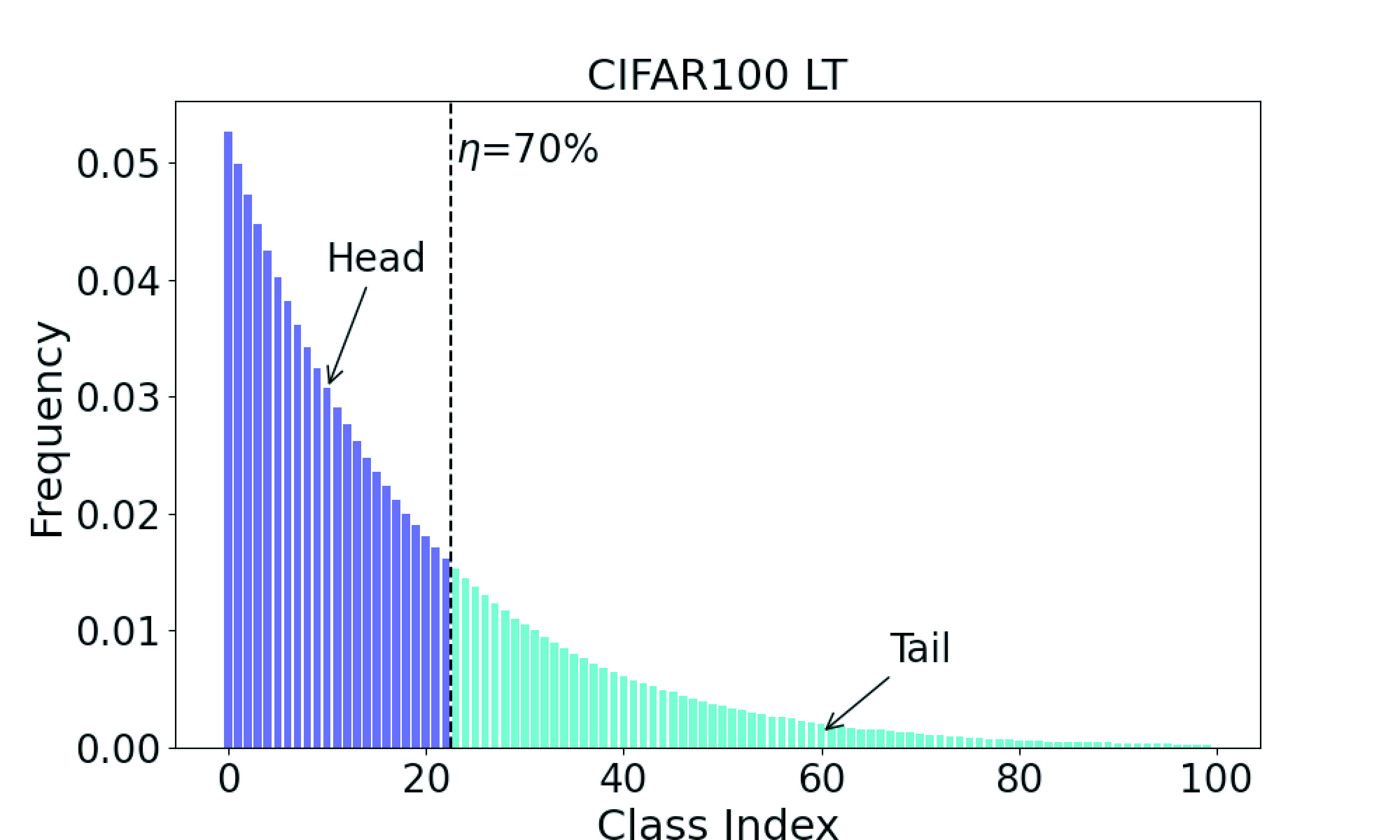}
        \caption{}\label{fig:distribution_lt}
    \end{subfigure}%
    \begin{subfigure}[t]{0.48\textwidth}
        \centering
        \includegraphics[width=\linewidth]{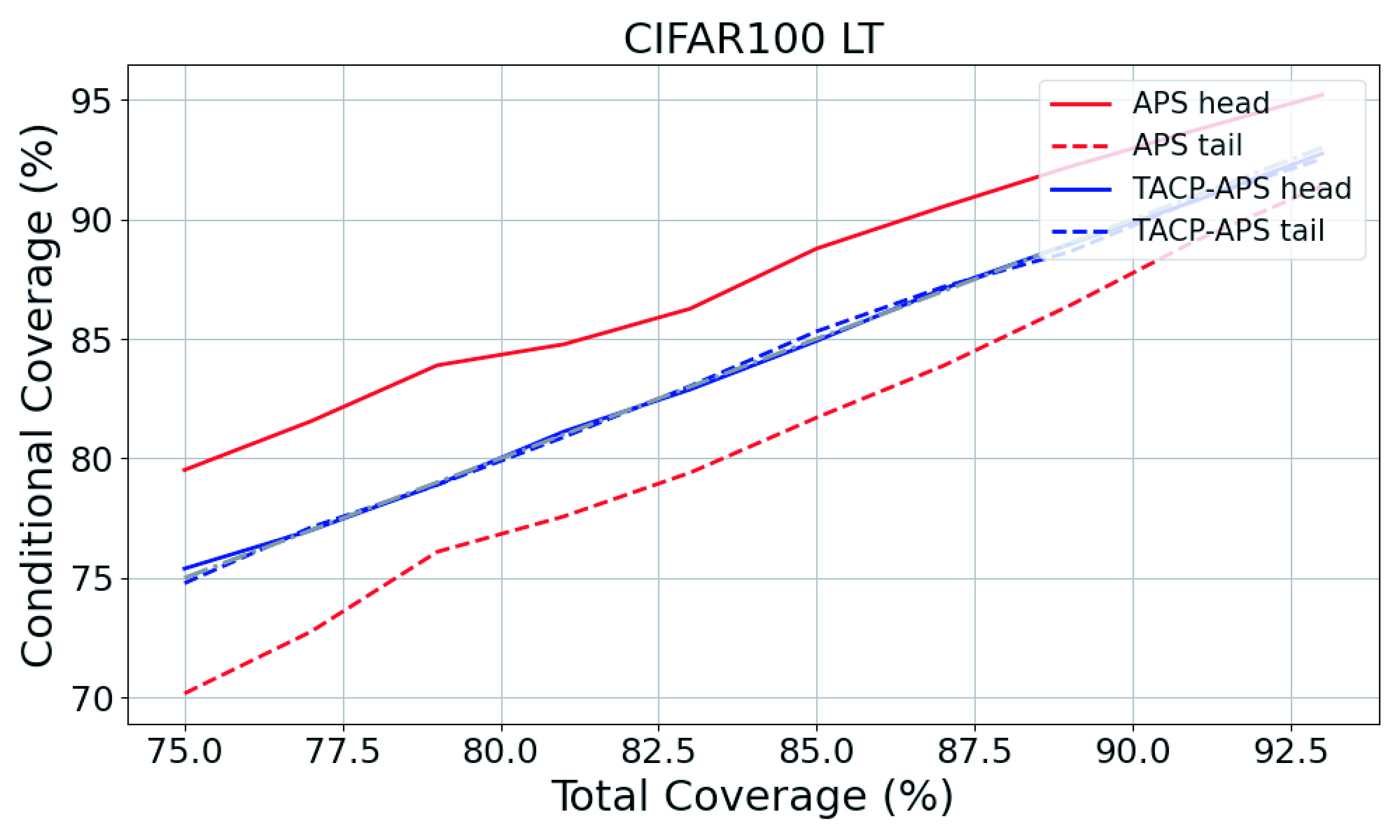}
        \caption{}
        \label{fig:head_tail_coverage}
    \end{subfigure}%
    \caption{Left: Head-tail partitioning at $\eta=70\%$. Right: Head- and tail-conditional coverage as a function of total coverage for the STANDARD and TACP methods using APS non-conformity score. Both figures are based on CIFAR100 LT with an imbalance factor $\mu=100$, which quantifies the ratio between the most and least frequent classes.}
    \label{illustrate}
\end{figure}

\section{Preliminaries}
\subsection{Conformal Prediction}
In conformal prediction, let the feature space be $\mathcal{X}\subseteq \mathbb{R}^d$ and the label space be $\mathcal{Y}=\{1,2,\dots,K\}$ with $K$ classes. We focus on the feature-label random variable $X\times Y\in \mathcal{X}\times \mathcal{Y}$, which has a joint distribution with density $p(\bm{x}, y)$ where $\bm{x},y$ are their realizations. The calibration samples $\{(X_i,Y_i)\}_{i=1}^n$ and the test point $(X_{n+1},Y_{n+1})$ are independently and identically drawn (i.i.d.) from the unknown distribution $p(\bm{x}, y)$ (the i.i.d. assumption can be relaxed to an exchangeability assumption \cite{vovk2005algorithmic}). Ordinary CP in classification tasks aims at constructing a set predictor $\mathcal{C}:\mathcal{X}\rightarrow 2^{\mathcal{Y}}$ such that for a new test point $(X_{n+1},Y_{n+1})$, the following coverage guarantee holds:
\begin{align}
\label{marginal_cov_eq}
    \mathbb{P}(Y_{n+1} \in \mathcal{C}(X_{n+1}))\geq 1-\alpha,
\end{align}
where $\alpha\in(0,1)$ is a predefined miscoverage level. 
The typical approach for generating such a prediction set is to design a non-conformity score function $s:\mathcal{X}\times\mathcal{Y}\rightarrow \mathbb{R}$ that quantifies the uncertainty of the pretrained model 
$\hat{\pi}:\mathcal{X}\rightarrow \Delta^{K-1}$ and then return a label set whose scores fall below some acceptance threshold. 
For example, split conformal prediction takes the acceptance threshold as the $1-\alpha$ quantile of the calibration non-conformity scores:
\begin{align*}
    \mathcal{C}(X_{n+1})=\{y: s(X_{n+1},y)\leq \hat{\tau}_{\alpha}\},
\end{align*}
where
    $\hat{\tau}_{\alpha} = \text{Quantile}\left(1-\alpha,\{s(X_i,Y_i)\}_{i=1}^n \right)$.
We can make the prediction set more conservative by defining  $\hat{\tau}_{\alpha}$ as the $\lceil (n+1)(1-\alpha)\rceil/n$ quantile and refer to this baseline as STANDARD. CP is a useful approach that provides a marginal coverage guarantee, i.e., Eq.~(\ref{marginal_cov_eq}), without requiring any assumptions on the model or data distribution.

\subsection{Class-Conditional Coverage}
However, marginal coverage only ensures that predictions are valid on average over the entire population, which often masks imbalanced errors across subgroups \citep{DBLP:conf/nips/DingABJT23,lofstrom2015bias}.
To address this, a stronger notion is class-conditional coverage, which imposes the coverage guarantee for each class $y \in \mathcal{Y}$ individually:
\begin{align}
\label{classwise_cov}
     \mathbb{P}(Y_{n+1} \in \mathcal{C}(X_{n+1})\mid Y_{n+1}=y)\geq 1-\alpha.
\end{align}
A direct approach to achieve Eq.~(\ref{classwise_cov}) is to split the calibration dataset by class and applying conformal prediction separately to each class \cite{DBLP:journals/ml/Vovk13}, denoted by CLASSWISE.

However, the per-class acceptance threshold estimators in CLASSWISE can be highly noisy for classes with limited samples, often causing overly large prediction sets~\cite{shi2013applications} (See Exp.3 in Section 7).

To mitigate the inefficiency of prediction sets, \citet{DBLP:conf/nips/DingABJT23} introduced CLUSTER, which groups “similar” classes by clustering based on quantiles of their non-conformity score distributions. Calibration samples are then assigned to these clusters, and standard CP is applied within each cluster.
Moreover, \citet{ShiGBD024} provided the Rank Calibrated Class-conditional CP (RC3P) method, which aims to achieve class-conditional coverage while improving prediction efficiency by calibrating the acceptance thresholds for each class using a label-rank calibration procedure.

Although CLUSTER and RC3P methods significantly improve efficiency while approximating class-conditional coverage, both methods require splitting the calibration set by clusters or classes and estimating acceptance thresholds within each subset. 
In the LT setting, the scarcity of tail-class samples leads to highly unreliable threshold estimates, thereby limiting the practicality of such methods.  This challenge is even more pronounced in the fully long-tail setting we consider, where the training, calibration, and test sets all follow long-tail label distributions.

\section{Methods}
\label{section_method}

\subsection{CP under LT Distributions}
\paragraph{Head-Tail Partition} In LT settings, a common approach is to partition the label space $\mathcal{Y}$ into \textit{head} classes $\mathcal{G}_h$ and \textit{tail} classes $\mathcal{G}_t$. 
Specifically, we define the head classes $\mathcal{G}_h$ as the smallest subset of labels whose cumulative class prior probability exceeds a pre-defined threshold $\eta \in(0,1)$:
\begin{align}
\label{headdef}
    \mathcal{G}_h = \mathop{\rm argmin}_{\mathcal{G}\subseteq \mathcal{Y}} |\mathcal{G}|,~~\text{s.t.} \sum_{y\in \mathcal{G}} p_{y}\geq \eta,
\end{align}
where $p_y$ denotes the class prior probability of class $y$.
The tail class group $\mathcal{G}_t$ is then defined as $\mathcal{Y}\setminus \mathcal{G}_h$. For illustration, Figure~\ref{fig:distribution_lt} shows the class distribution of a long-tail CIFAR100, where classes are sorted by empirical frequency, and the head group corresponds to the smallest subset whose cumulative frequency exceeds $\eta$.

\paragraph{STANDARD} 
Standard CP methods applied to long-tail settings often exhibit substantial disparity between head and tail coverage. To quantify this effect, we evaluate APS-based Split CP on CIFAR100-LT under various target coverages $1-\alpha$. The details of experiments can be found in Appendix~\ref{Appendix_headtail_exp}. The results in Figure~\ref{fig:head_tail_coverage} show that head classes are consistently over-covered, with conditional coverage lying above the target, while tail classes are significantly under-covered, leading to a large head–tail gap.
The large imbalance of the label distribution causes this phenomenon because standard CP does not differentiate between classes, leading to frequent exclusion of minority classes from prediction sets. This observation motivates our approach to enforce fairer coverage across head and tail classes.

\subsection{Tail-Aware Conformal Prediction (TACP)}
To explicitly mitigate this head–tail coverage imbalance and improve the robustness of conformal prediction under long-tail settings, we propose the \textbf{T}ail-\textbf{A}ware \textbf{C}onformal \textbf{P}rediction (TACP) method.
TACP adaptively adjusts the label-ranking penalty by leveraging head–tail partition information, ensuring that prediction sets remain informative while reducing systematic coverage gaps between head and tail classes. Formally, the definition of TACP is given by:
\begin{align}
\label{tacp}
    s_{\text{TACP}}(\bm{x},y) = s(\bm{x},y) + \lambda\cdot\mathbb{I}(y \in \mathcal{G}_h)\cdot(o_{\bm{x}}(y)-k_r)^{+},
\end{align}
where $\lambda\in\mathbb{R}^{+}$ and $k_r\in\mathbb{N}$ are hyperparameters; $()^+$ denotes the ReLU function; $\mathbb{I}$ is the indicator function; and $o_{\bm{x}}(y) = \left| \left\{ y' \in \mathcal{Y} : \hat{\pi}_{y'}(\bm{x}) \geq \hat{\pi}_y(\bm{x}) \right\} \right|$ denotes the ranking of $y$ based on the estimated class posterior probability $\hat{\pi}(\bm{x})$.
Here, $s(\bm{x},y)$ can be any non-conformity scores. For example, the TACP-LAC method with $\lambda=1$ and $k_r=2$ can be derived by setting $s(\bm{x},y)$ as one minus estimated class posterior probability:
\begin{align*}
&s_{\text{TACP-LAC}}(\bm{x},y) = \underbrace{1-\hat{\pi}_y(\bm{x})}_{\text{LAC Score}}+\underbrace{\mathbb{I}(y \in \mathcal{G}_h)\cdot(o_{\bm{x}}(y)-2)^{+}}_{\text{Selective Rank Regularization}}.
\end{align*}

The key insight behind TACP is to utilize the LT information and selectively penalize label rankings for head, thereby reducing head-conditional coverage. To maintain the target marginal coverage level, this has the effect of increasing the tail-conditional coverage, thus narrowing the head-tail coverage gap. 
The label prediction set is then constructed:
\begin{align}
\label{TACP_set}
    \mathcal{C}_{\text{TACP}}(\bm{x}_{n+1}) :=\{y\in \mathcal{Y}: s_{\text{TACP}}(\bm{x}_{n+1},y)\leq \hat{q}_\alpha \}
\end{align}
where $\hat{q}_\alpha$ is the $\lceil
 (1-\alpha)(n+1)\rceil/n$ quantile of the calibration scores $\{s_{\text{TACP}}(\bm{X}_i,Y_i)\}_{i=1}^n$. Moreover, our TACP method enjoys the standard marginal coverage property:
\begin{theorem}[TACP Coverage Guarantee]
\label{marginal_cov}
 If $\{(X_i,Y_i)\}_{i=1}^n$ are i.i.d, then for any new i.i.d. draw $(X_{n+1},Y_{n+1})$ 
 \begin{align*}
     \mathbb{P}(Y_{n+1}\in \mathcal{C}_{\text{TACP}}(X_{n+1}))\geq 1-\alpha
 \end{align*}
 for the conformal prediction set $\mathcal{C}_{\text{TACP}}$ constructed in Eq.(\ref{TACP_set}). 
  Moreover, if we assume additionally a uniform random variable $u$ to ensure the scores $\{s_{\text{TACP}}(X_i,Y_i)\}_{i=1}^{n+1}$ are almost surely distinct,
\begin{align*}
    s_{\text{TACP}}(X_i,Y_i) = \;\,  s(X_i,Y_i) + {} \lambda\cdot\mathbb{I}(Y_i \in \mathcal{G}_h)\cdot(o_{X_i}(Y_i)-k_r+u)^{+}
\end{align*}
then the following upper bound holds
 \begin{align*}
       \mathbb{P}(Y_{n+1}\in \mathcal{C}_{\text{TACP}}(X_{n+1}))\leq (1-\alpha)+\frac{1}{n+1}.
 \end{align*}
\end{theorem}
We now present a theorem showing the effectiveness of TACP in reducing head–tail conditional coverage disparity.
\begin{theorem}[Improved Coverage Gap]
Let $E_{xy}$ be the event that the calibration dataset is fixed as $\{(X_i,Y_i)\}_{i=1}^n=\{(\bm{x}_1,y_1),\dots, (\bm{x}_n,y_n)\}$. Suppose $\mathcal{C}_{\text{TACP}}$ and $\mathcal{C}_{\text{STD}}$ are the prediction sets constructed by TACP and the STANDARD method, respectively, based on the same non-conformity score. Then there exists $k_r$, such that the group-conditional coverage gap between head and tail groups satisfies:
\begin{align*}
    &P(Y_{n+1} \in \mathcal{C}_{\text{TACP}}(X_{n+1}) \mid Y_{n+1} \in \mathcal{G}_h, E_{xy})
    - P(Y_{n+1} \in \mathcal{C}_{\text{TACP}}(X_{n+1}) \mid Y_{n+1} \in \mathcal{G}_t, E_{xy}) \\ 
\leq &P(Y_{n+1} \in \mathcal{C}_{\text{STD}}(X_{n+1}) \mid Y_{n+1} \in \mathcal{G}_h, E_{xy})
-P(Y_{n+1} \in \mathcal{C}_{\text{STD}}(X_{n+1}) \mid Y_{n+1} \in \mathcal{G}_t, E_{xy}).
\end{align*}   
\label{theorem2}
\end{theorem}
By introducing the selective rank regularization term, the adjusted acceptance threshold increases the probability of correctly including tail classes, while the head-conditional coverage is controlled by properly choosing parameter $\lambda$ and $k_r$, thereby narrowing the coverage gap compared with STANDARD method using the same non-conformity score. A complete proof is given in Appendix~\ref{Appendix_proof}.

\section{Soft TACP}
While TACP addresses coverage imbalance between head and tail groups,
a more fine-grained goal is to balance coverage across individual classes in LT settings~\cite{lofstrom2015bias}.
In this section, we extend TACP beyond the binary head–tail partition to promote coverage balance across all classes, i.e., moving towards class-conditional coverage rather than merely head–tail coverage.

In the original TACP formulation Eq.~\eqref{tacp}, the indicator term plays an crucial role in penalizing head classes, thereby reducing the head–tail coverage gap as guaranteed by Theorem~\ref{theorem2}.
However, this indicator term may be a hindrance when we want to move beyond the head-tail setting: according to the definition of $\mathcal{G}_{h}$ in Eq.~\eqref{headdef}, it naturally entails the head-tail partition, which may not align well with the goal of class-conditional coverage balance.

To address this limitation, we propose removing the dependence on the head–tail partition $\mathcal{G}_{h}$ and replacing the binary indicator with a soft, class-aware weighting scheme. This allows penalties to adapt continuously to class prior probabilities rather than enforcing a rigid head–tail split, enabling more nuanced control over coverage across classes.
\paragraph{Soft TACP} Building on this idea, we introduce the \textbf{s}oft \textbf{T}ail-\textbf{A}ware \textbf{C}onformal \textbf{P}rediction (sTACP) method, which generalizes TACP by continuously reweighting penalties according to the estimated class prior $\hat{p}(y)\in[0,1]$. Instead of the original hard $0$–$1$ indicator, sTACP employs a smooth weighting function that reflects both the long-tail structure and class-specific information:

\begin{align*}
s_{\text{sTACP}}(\bm{x},y) = s(\bm{x},y) +\lambda \cdot \hat{p}(y)\cdot(o_{\bm{x}}(y)-k_r)^{+},
\end{align*}
where $\lambda\in\mathbb{R}^{+}$ and $k_r\in \mathbb{N}$ are hyperparameters. This enables the label ranking penalty to adapt more smoothly to the long-tail distribution without dependency to the head-tail partition. The efficiency of this extension is experimentally validated in Section \ref{section_class-conditional}.

\section{Head-Tail Experiments}
\label{section_headtail_exp}
We evaluate the head- and tail-conditional coverages of STANDARD, Partition-Wise, and TACP methods on several versions of long-tail CIFAR100 and ImageNet datasets using four representative non-conformity scores. 

\subsection{Experimental Setup}
\paragraph{Datasets and Models}
In all experiments, the training, calibration, and test splits follow a long-tail distribution, ensuring that the conformal calibration process operates under highly imbalanced conditions.
For CIFAR100 \cite{CIFAR}, we construct long-tail variants following \citet{CIFAR100LT} with imbalance factors $\mu \in \{50, 100\}$, where $\mu = \frac{\max_i n_i}{\min_j n_j}$ and $n_k$ denotes the number of samples from class $k$. For ImageNet \cite{ImageNet}, we create long-tail subsets by sampling from a Pareto distribution with power parameter $\rho \in [0.3, 0.6]$ following \citet{Imagenetlt2019}. All methods are evaluated using pretrained models from \citet{MetaSAugmodel}. Table~\ref{modelsdatascifariamgenet} summarizes the datasets and models used in our experiments.

\begin{table}[!t]
    \fontsize{9pt}{12pt}\selectfont
    \centering
    \begin{tabular}{c|c|c|c|c}
    \toprule
Statistics&  \multicolumn{2}{c|}{CIFAR100 LT} & \multicolumn{2}{c}{ImageNet LT}   \\
\midrule
$\#$Label &100 &100& 1000&1000\\
\midrule
Pareto Power ($\rho$) & --- &--- & 0.3&0.6    \\
\midrule
Imbalance Rate ($\mu$) &  $100$ & $50$  & $8.3$ & $25$\\
\midrule
$\#$Train & 10626& 12354& 115846 &115846 \\
\midrule
$\#$Test & 5478& 6329 & 8442 & 2453\\
\midrule
Prediction Acc & 59.80$\%$ & 56.14$\%$& 48.21$\%$ &50.75$\%$\\
\bottomrule
    \end{tabular}
    \caption{Specification of datasets. We evaluate the pretrained models' performance on long-tail test datasets.}
    \label{modelsdatascifariamgenet}
\end{table}

\paragraph{Baselines}
We compare three methods: \textbf{STANDARD}, \textbf{Partition-Wise}, and \textbf{TACP}. The Partition-Wise approach aims to reduce the head–tail coverage gap by dividing the calibration set into head and tail subsets and applying conformal prediction independently within each group. Additional details on these methods, along with the hyperparameter tuning procedure for TACP, are provided in Appendix~\ref{Appendix_headtail_exp}.

\paragraph{Non-conformity Scores} We consider four scores: \textbf{APS}, which approximates X-conditional coverage \citep{DBLP:conf/nips/RomanoSC20}; \textbf{RAPS}, a regularized variant of APS to generate smaller prediction set \citep{RAPS2021}; \textbf{LAC}, defined as one minus the softmax output \citep{sadinle2019least}; \textbf{TOPK}, which yields uniformly prediction set for all test samples \citep{RAPS2021}. 
Definitions of the scores and the hyperparameters used in our experiments are provided in 
Appendix~\ref{Appendix_headtail_exp}. \\
\paragraph{Evaluation Metric}
Assume $\mathcal{I}_h=\{i \in [N_{\text{test}}] : y'_i \in \mathcal{G}_h\}$ and $\mathcal{I}_t=\{i \in [N_{\text{test}}] : y'_i \in \mathcal{G}_t\}$ %
be the indices of test examples $\{(\bm{x}'_j,y'_j)\}_{j=1}^{N_{\text{test}}}$ with head and tail classes, respectively. We give the metrics of head-tail coverage gap and average set size as follows:
\begin{align*}
    &\text{Cov-head}=100\times \frac{\sum_{i\in \mathcal{I}_h} \mathbb{I}( y'_i\in \mathcal{C}(\bm{x}'_i))}{|\mathcal{I}_h|},   ~\text{Cov-tail}=100\times \frac{\sum_{i\in \mathcal{I}_t} \mathbb{I}( y'_i\in \mathcal{C}(\bm{x}'_i))}{|\mathcal{I}_t|},\\
&\text{CovGap-HT} = |\text{Cov-head}- \text{Cov-tail}|,~\text{AvgSize} =\frac{1}{N_{\text{test}}}\sum \nolimits_{i=1}^{N_{\text{test}}}|\mathcal{C}(\bm{x}_{i})|.
\end{align*}

\subsection{Results}
We start with a brief summary of the experiments. In Experiment 1 (\textbf{Exp.1}), we compare the Cov-head and Cov-tail of STANDARD and TACP methods using four non-conformity scores on CIFAR100 LT and ImageNet LT datasets, showing that TACP can significantly reduce the CovGap-HT. In Experiment 2 (\textbf{Exp.2}), we evaluate the AvgSize and CovGap-HT of STANDARD, Partition-Wise, and TACP methods on several versions of CIAFR100 LT and ImageNet LT.

\paragraph{Exp.1: Cov-head and Cov-tail}
To evaluate the coverage disparity between head and tail classes, we compare the STANDARD and TACP methods on CIFAR100-LT ($\mu = 100$) and ImageNet-LT ($\rho = 0.3$) at a fixed partition $\eta = 50\%$, using four non-conformity scores. Specifically, we examine the results of Cov-head and Cov-tail.

Figure~\ref{headtail_total_cov} illustrates the behavior of different methods as the total coverage level $1-\alpha$ varies from $80\%$ to $99\%$ in steps of $1\%$. We observe that under LT label distributions, the STANDARD method suffers from a coverage bias between head and tail classes, with the Cov-head (solid curves) for all scores consistently lying above the diagonal and the Cov-tail (dashed curves) remaining below it. In contrast, the TACP method effectively narrows the gap between the Cov-head and the Cov-tail, as shown by the almost overlapping dashed and solid lines. 
The results show that, compared with STANDAND, TACP can significantly alleviate the head-tail coverage imbalance and mitigate the under coverage of tail classes across all scores and datasets in LT settings.

\begin{figure*}[tbp]
    \centering
    \begin{subfigure}[t]{0.46\textwidth}
        \centering
    \includegraphics[width=\linewidth]{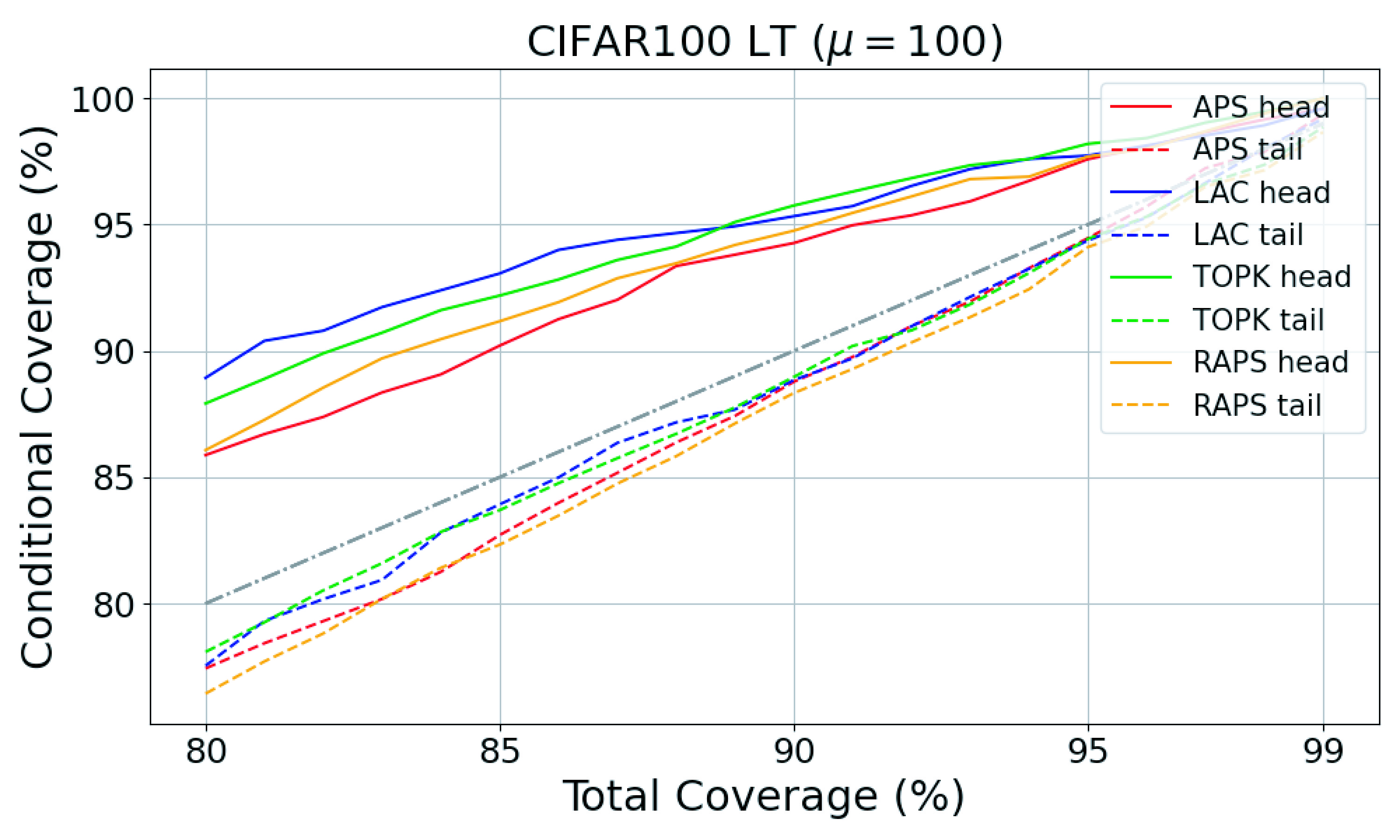}
    \end{subfigure}
    \begin{subfigure}[t]{0.46\textwidth}
        \centering
    \includegraphics[width=\linewidth]{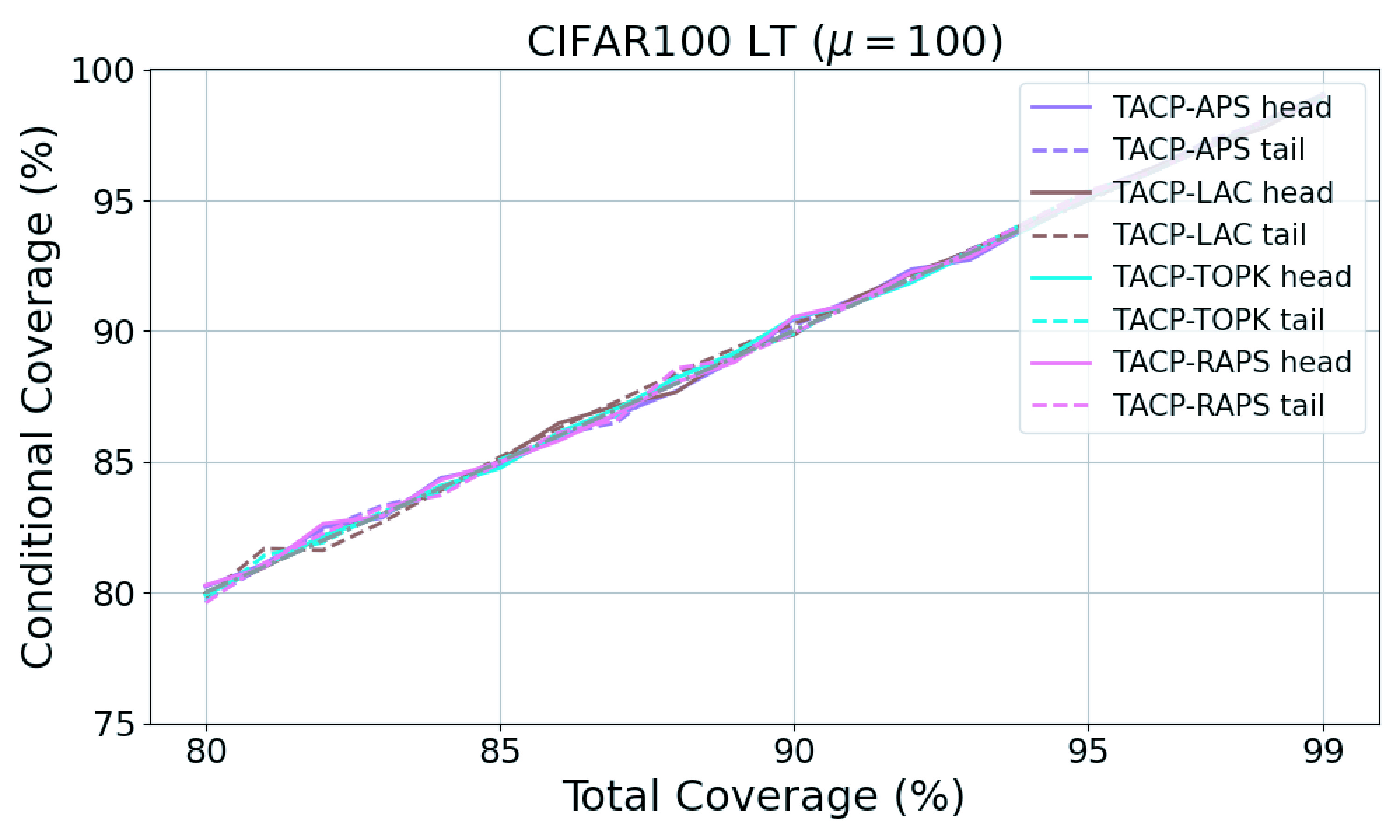}
    \end{subfigure}
    \begin{subfigure}[t]{0.46\textwidth}
        \centering
    \includegraphics[width=\linewidth]{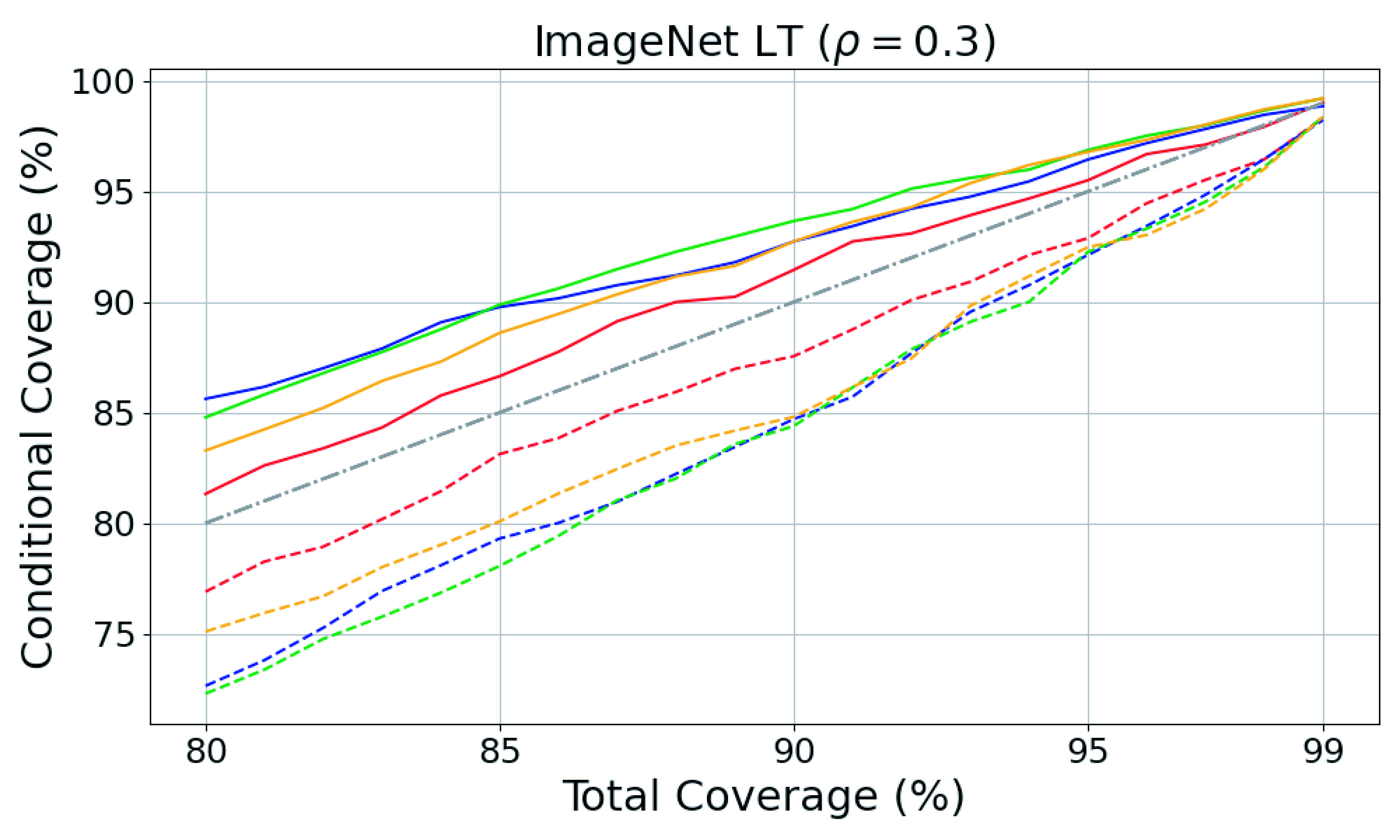}
    \end{subfigure}
    \begin{subfigure}[t]{0.46\textwidth}
        \centering
    \includegraphics[width=\linewidth]{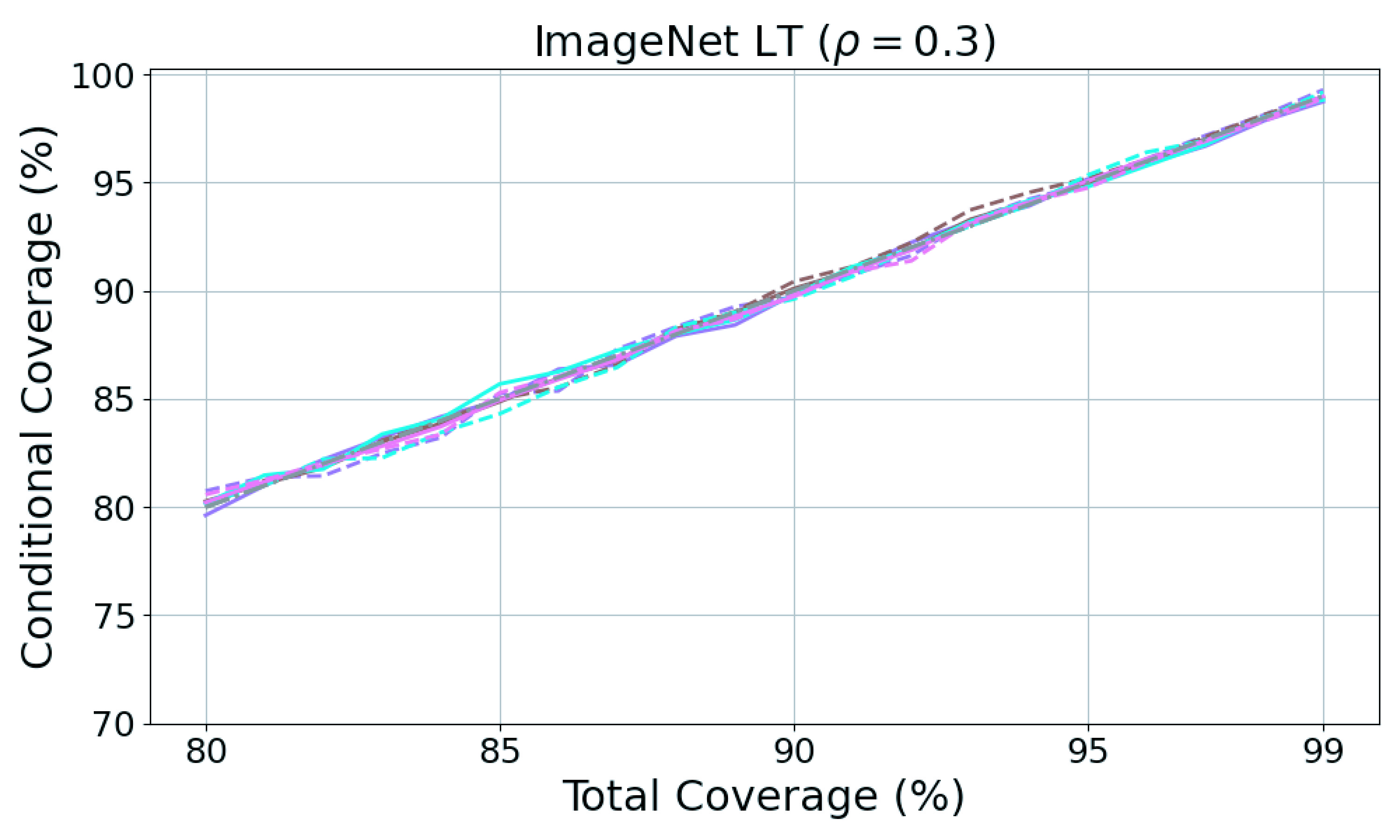}
    \end{subfigure}
    \caption{Head- and tail-conditional coverage versus total coverage for STANDARD (left column) and TACP (right column) across CIFAR100-LT ($\mu = 100$) and ImageNet-LT ($\rho = 0.3$). Results are shown for four non-conformity scores at $\eta = 50\%$. The gray dashed line indicates the ideal coverage (diagonal).}
    \label{headtail_total_cov}
\end{figure*}

\paragraph{Exp.2: CovGap-HT and AvgSize}
We compare the STANDARD (STA), Partition-Wise (PW), and TACP methods in terms of CovGap-HT and AvgSize across multiple LT datasets: CIFAR100-LT with $\mu \in \{50, 100\}$ and ImageNet-LT with $\rho \in \{0.3, 0.6\}$, evaluated at a miscoverage level $\alpha = 10\%$ and head–tail partition $\eta = 50\%$. The results are summarized in Table~\ref{CovGap_HT_table}.

We can observe that TACP consistently achieves a substantial reduction in CovGap-HT without compromising prediction set informativeness compared to the STANDARD method. For example, applying TOPK as base score on ImageNet LT ($\rho=0.6$), the CovGap-HT of the STANDARD-TOPK method is $8.83\%$, whereas TACP-TOPK reduces it to $1.02\%$ with only a slight increase in AvgSize from $23.15$ to $25.63$. 
Overall, these results demonstrate the effectiveness of TACP in narrowing CovGap-HT while maintaining efficient prediction sets.
Additional results over a broader range of $\eta$ (Appendix~\ref{Appendix_headtail_exp}) further confirm the robustness of TACP.

\begin{table*}[tbp]
    \centering
    \scalebox{0.85}{
    \begin{tabular}{c|c|cc|cc|cc|cc}
\toprule
\multirow{2}{*}{Score} & \multirow{2}{*}{Method} & \multicolumn{2}{c|}{$\rho=0.6$}& \multicolumn{2}{c|}{$\rho=0.3$}& \multicolumn{2}{c|}{$\mu=100$}& \multicolumn{2}{c}{$\mu=50$}\\
\cmidrule{3-4} \cmidrule{5-6} \cmidrule{7-8} \cmidrule{9-10}
&&CovGap-HT & AvgSize&CovGap-HT & AvgSize&CovGap-HT & AvgSize&CovGap-HT & AvgSize \\
\midrule
\multirow{3}{*}{APS}
&STA
&2.18$\pm$1.21&36.43$\pm$2.21&2.79$\pm$0.85&35.29$\pm$2.29
&4.81$\pm$0.89&9.93$\pm$0.45&3.87$\pm$0.81&8.58$\pm$0.32
\\
&PW
&2.01$\pm$1.29&38.10$\pm$2.65&1.43$\pm$1.20&38.35$\pm$2.63
&1.38$\pm$1.19&10.28$\pm$0.43&1.22$\pm$0.84&8.79$\pm$0.25
\\
&TACP
&\textbf{1.11$\pm$0.71}&33.98$\pm$2.40&\textbf{0.76$\pm$0.60}&34.07$\pm$2.22
&\textbf{0.78$\pm$0.64}&10.25$\pm$0.46&\textbf{0.76$\pm$0.59}&8.55$\pm$0.26
\\
\midrule
\multirow{3}{*}{LAC}
&STA
&7.45$\pm$1.33&15.36$\pm$0.96&7.50$\pm$0.73&14.81$\pm$1.00
&6.35$\pm$1.01&7.11$\pm$0.29&4.88$\pm$0.91&6.16$\pm$0.26
\\
&PW
&1.96$\pm$1.28&17.22$\pm$1.30&1.25$\pm$1.25&17.72$\pm$1.20
&1.65$\pm$1.18&7.44$\pm$0.35&1.56$\pm$1.13&6.36$\pm$0.25
\\
&TACP
&\textbf{1.18$\pm$0.94}&17.92$\pm$1.40&\textbf{0.81$\pm$0.67}&19.16$\pm$1.91
&\textbf{0.77$\pm$0.59}&7.86$\pm$0.36&\textbf{0.66$\pm$0.52}&6.47$\pm$0.25
\\
\midrule
\multirow{3}{*}{TOPK}
&STA
&8.83$\pm$1.69&23.15$\pm$1.45&9.45$\pm$0.89&20.77$\pm$1.22
&6.47$\pm$0.84&10.90$\pm$0.50&5.51$\pm$0.82&8.58$\pm$0.36
\\
&PW
&1.87$\pm$1.53&24.06$\pm$2.46&1.31$\pm$0.98&24.25$\pm$1.98
&1.33$\pm$1.16&11.01$\pm$0.47&1.27$\pm$0.97&8.71$\pm$0.40
\\
&TACP
&\textbf{1.02$\pm$0.84}&25.63$\pm$2.44&\textbf{0.58$\pm$0.48}&26.26$\pm$2.24
&\textbf{0.82$\pm$0.61}&11.88$\pm$0.57&\textbf{0.77$\pm$0.57}&9.38$\pm$0.39
\\
\midrule
\multirow{3}{*}{RAPS}
&STA
&7.76$\pm$1.25&16.49$\pm$1.30&7.34$\pm$0.71&15.98$\pm$1.04
&6.34$\pm$0.80&7.67$\pm$0.52&4.58$\pm$0.75&6.62$\pm$0.18
\\
&PW
&1.91$\pm$1.43&20.32$\pm$1.56&1.27$\pm$1.00&20.57$\pm$1.48
&1.49$\pm$1.16&8.93$\pm$0.50&1.24$\pm$0.88&6.48$\pm$0.23
\\
&TACP
&\textbf{1.10$\pm$0.87}&20.67$\pm$1.58&\textbf{0.61$\pm$0.41}&21.76$\pm$2.03
&\textbf{0.82$\pm$0.70}&9.62$\pm$0.68&\textbf{0.80$\pm$0.58}&6.76$\pm$0.24
\\
\bottomrule
    \end{tabular}
    }
    \caption{We report the average results$\pm$standard deviation over 100 different trails at $\alpha=10\%$ and $\eta=50\%$. Bold indicates the optimal CovGap-HT among all approaches for each non-conformity score.}
    \label{CovGap_HT_table}
\end{table*}

\section{Class-conditional Experiments}
\label{section_class-conditional}
In this section, we compare the class-conditional results of STANDARD, CLASSWISE, CLUSTER, RC3P, and sTACP on ImageNet-LT across four non-conformity scores.
\subsection{Experimental Setup}
\textbf{Baselines}
 We use the same datasets and pretrained models described in Section~\ref{section_headtail_exp}. 
We evaluated five methods: STANDARD, CLASSWISE, CLUSTER, RC3P, and sTACP.
Details of these baselines and hyperparameters for finetuning of sTACP are provided in Appendix~\ref{Appendix_classwise-exp}.\\
\textbf{Evaluation Metric}
We evaluate the performance on three metrics: \textbf{coverage}, which measures the empirical marginal coverage; \textbf{set size} (AvgSize), which reflects the efficiency of the prediction set; \textbf{class-conditional coverage gap} (CovGap), which quantifies coverage disparities across classes.
Details are provided in Appendix~\ref{Appendix_classwise-exp}.
\subsection{Results}
We provide a brief overview of the experimental setup. Experiment 3 (\textbf{Exp.3}) reports CLASSWISE results on ImageNet-LT using four different scores. Experiment 4 (\textbf{Exp.4}) compares STANDARD, CLUSTER, RC3P, and sTACP across two long-tail variants of ImageNet-LT. Experiment 5 (\textbf{Exp.5}) examines performance on ImageNet-LT ($\rho = 0.6$) under a different target coverage, with $\alpha = 5\%$.

\paragraph{Exp.3: CLASSWISE Method}
To examine the limitations of the CLASSWISE (CW) method under LT settings, we conduct an experiment on ImageNet-LT ($\rho=0.6$) at $\alpha=10\%$ using four different non-conformity scores. As shown in Table~\ref{classwise_method_imagenet}, the CW approach frequently outputs the entire label set $\mathcal{Y}$, leading to a trivial average coverage of $100\%$. Such excessively large prediction sets are uninformative and of limited value for downstream tasks.
\begin{table}[!t]
    \centering       
    \fontsize{9pt}{12pt}\selectfont
    \begin{tabular}{cccc}
       \toprule
Score &CovGap &Coverage& AvgSize \\
\midrule
APS&10.00$\pm$0.02&99.46$\pm$0.31&944.48$\pm$2.42\\
LAC&9.98$\pm$0.01&99.51$\pm$0.22&942.51$\pm$1.43\\
TOPK&9.98$\pm$0.02&99.56$\pm$0.27&943.16$\pm$1.23\\
RAPS&10.00$\pm$0.02&99.54$\pm$0.33&943.15$\pm$1.22\\
\bottomrule
    \end{tabular}
 \caption{Performance of CLASSWISE method on ImageNet LT ($\rho=0.6$) at $\alpha=10\%$.}
    \label{classwise_method_imagenet}
\end{table}

\paragraph{Exp.4: sTACP Method} 
We compare STANDARD (STA), CLUSTER (CLUS), RC3P, and sTACP methods on two versions of ImageNet LT $\rho\in\{0.3,0.6\}$ at $\alpha=10\%$.
As shown in Table~\ref{imagenet-classwise}, sTACP consistently reduces CovGap with only a modest increase in AvgSize compared with other baselines.
For instance, the CovGap of STA and sTACP decreases from $19.00\%$ to $15.86\%$ when applying APS as the base score on ImageNet LT ($\rho=0.3$), whereas their AvgSize remains similar.
In summary, sTACP method consistently outperforms the baselines under LT distributions, achieving smaller class-conditional coverage gaps with the same score.

\begin{table*}[!t]
  \centering
   \fontsize{9pt}{12pt}\selectfont
  \begin{tabular}{c|c|ccc|ccc}
   \toprule
   \multirow{2}{*}{Score} &   \multirow{2}{*}{Method}& \multicolumn{3}{c|}{$\rho=0.6$}&  \multicolumn{3}{c}{$\rho=0.3$}   \\
    \cmidrule(r){3-5} \cmidrule(r){6-8} 
& & CovGap & Coverage & AvgSize & CovGap & Coverage &AvgSize \\
\midrule
\multirow{5}{*}{APS}&STA&19.00$\pm$0.00&89.78$\pm$0.01&35.60$\pm$2.27
&13.37$\pm$0.35&90.07$\pm$0.61&38.39$\pm$1.54
\\
&CLUS
&17.84$\pm$0.00&89.66$\pm$0.01&37.54$\pm$1.43
&13.41$\pm$0.00&90.62$\pm$0.00&42.80$\pm$1.28
\\
&RC3P
&18.99$\pm$0.61&89.31$\pm$0.65&30.20$\pm$1.88
&14.21$\pm$0.31&89.45$\pm$0.45&43.50$\pm$2.22\\
&sTACP
&\textbf{15.86$\pm$0.81}&89.53$\pm$1.12&36.52$\pm$3.66
&\textbf{12.76$\pm$0.31}&90.04$\pm$0.61&35.48$\pm$2.08
\\
\midrule
\multirow{5}{*}{LAC}& STA&18.87$\pm$0.01&90.20$\pm$0.01&15.30$\pm$0.81
&14.35$\pm$0.36&89.90$\pm$0.63&16.94$\pm$0.67\\
&CLUS
&17.80$\pm$0.00&89.10$\pm$0.00&16.19$\pm$0.05
&14.05$\pm$0.00&90.82$\pm$0.00&19.59$\pm$0.10\\
&RC3P
&19.17$\pm$0.59&89.19$\pm$0.70&30.45$\pm$2.09
&14.24$\pm$0.32&89.55$\pm$0.45&43.10$\pm$2.00\\
&sTACP
&\textbf{15.98}$\pm$0.78&89.54$\pm$1.07&37.57$\pm$3.40
&\textbf{12.66$\pm$0.28}&90.09$\pm$0.59&41.72$\pm$1.72\\
\midrule
\multirow{5}{*}{TOPK}& STA&  18.87$\pm$0.01&90.12$\pm$0.01&23.28$\pm$1.23
&14.38$\pm$0.41&89.96$\pm$0.61&23.55$\pm$1.02\\
&CLUS
&17.59$\pm$0.00&89.57$\pm$0.00&24.24$\pm$0.15
&13.89$\pm$0.00&91.61$\pm$0.00&27.96$\pm$0.20\\
&RC3P 
&18.99$\pm$0.61&89.39$\pm$0.67&31.01$\pm$2.15
&14.21$\pm$0.31&89.67$\pm$0.47&45.62$\pm$2.32\\
&sTACP
&\textbf{15.83$\pm$0.77}&89.55$\pm$1.08&38.99$\pm$3.54
&\textbf{12.80$\pm$0.28}&90.05$\pm$0.58&37.38$\pm$1.71
\\
\midrule
\multirow{5}{*}{RAPS}&STA
&18.45$\pm$0.93&89.98$\pm$1.07&16.49$\pm$1.30
&14.25$\pm$0.33&89.98$\pm$0.60&18.32$\pm$1.04\\
&CLUS
&17.69$\pm$0.00&90.13$\pm$0.00&18.53$\pm$0.09
&13.92$\pm$0.00&90.99$\pm$0.00&23.54$\pm$1.25
\\
&RC3P
&18.99$\pm$0.61&89.35$\pm$0.66&30.92$\pm$2.14
&14.19$\pm$0.31&89.70$\pm$0.46&45.56$\pm$2.31\\
&sTACP
&\textbf{16.04$\pm$0.82}&89.54$\pm$1.10&33.68$\pm$3.45
&\textbf{12.90$\pm$0.31}&90.04$\pm$0.61&32.85$\pm$1.65
\\
\bottomrule
  \end{tabular}
  \caption{Performance on ImageNet LT at $\alpha=10\%$. We report the average results$\pm$standard deviation for 100 different trials. Bold indicates the optimal CovGap for each non-conformity score.}  
\label{imagenet-classwise}
\end{table*}

\paragraph{Exp.5: Evaluation at Stricter Coverage} To further evaluate robustness under stricter coverage requirements, we apply the APS score to all baselines at $\alpha = 5\%$ on ImageNet-LT ($\rho = 0.6$). Table~\ref{ImageNetClasswise_alpha_0.05} shows that CW generates uninformative prediction sets with an average size of 984. RC3P fails in this setting, likely due to the instability of class-specific top-$k$ error estimates under severe data imbalance.
For CLUSTER, the combination of long-tail distributions and a high coverage requirement causes the clustering to collapse into a single “null” cluster, effectively reducing the method to STANDARD and rendering its results uninformative.

In contrast, sTACP achieves substantially smaller CovGap while maintaining efficient prediction sets.
Additional results using the LAC at $\alpha = 5\%$ are reported in Appendix~\ref{Appendix_classwise-exp}.

\begin{table}[!t]
  \centering
   \fontsize{9pt}{12pt}\selectfont
  \begin{tabular}{cccc}
   \toprule
Method & CovGap & Coverage & AvgSize \\
    \midrule
STA&9.62$\pm$0.01 &94.88$\pm$0.01& 67.49$\pm$5.62\\
CW&5.00$\pm$0.01&99.90$\pm$0.16& 983.53$\pm$2.86\\
CLUS&{———}&{———}&{———}\\
RC3P&15.28$\pm$0.69&89.57$\pm$0.66&31.80$\pm$2.28\\
sTACP&8.08$\pm$0.00&95.30$\pm$0.00&76.83$\pm$0.17\\
\bottomrule
  \end{tabular}
  \caption{Performance with base score APS on ImageNet LT ($\rho=0.6$) at $\alpha=5\%$.}  
    \label{ImageNetClasswise_alpha_0.05}
\end{table}

\section{Conclusion}
In this paper, we propose TACP to achieve a more balanced performance across head and tail labels under long-tail distributions. By selectively penalizing label ranks for head labels, TACP reduces the prediction set size for head labels, which indirectly improves the tail-conditional coverage, thus narrowing the conditional coverage gap. 
We also provide a theoretical analysis demonstrating TACP's effectiveness in reducing this gap. We further proposed sTACP method to reduce the class-conditional coverage gap in long-tail settings. Experimental results demonstrate that our proposed methods consistently improve head-tail and class-conditional coverage gap across several benchmarks.

\newpage
\appendix

\section{Details of Head-Tail Experiments}
\label{Appendix_headtail_exp}

We implemented all the methods by Pytorch \cite{pytorch} and conducted all the experiments on NVIDIA GeForce RTX 4090 GPUs.

\subsection{Partition-Wise (PW) Method} 
The PW method mitigates the coverage imbalance by partitioning the calibration set $\{(X_i,Y_i)\}_{i=1}^n$ into two subsets according to Eq.~\eqref{headdef}: a head set containing samples from head classes and a tail set containing samples from tail classes, and then computing the acceptance thresholds for each set,
\begin{align*}
    &\hat{\tau}_h = \text{Quantile}(\frac{\lceil(1-\alpha)(1+n_h)\rceil}{n_h},\{s(X_i,Y_i)\}_{i\in\mathcal{I}_h}),\\
    &\hat{\tau}_t = \text{Quantile}(\frac{\lceil(1-\alpha)(1+n_t)\rceil}{n_t},\{s(X_i,Y_i)\}_{i\in\mathcal{I}_t}),
\end{align*}
where $\mathcal{I}_h$ and $\mathcal{I}_t$ are the indices of calibration samples from head and tail classes, with corresponding sample sizes $n_h$ and $n_t$. The conformal prediction set is then defined as:
\begin{align*}
   & \mathcal{C}_{\text{PW}}(X_{n+1})=\{y: s(X_{n+1},y) \leq \hat{\tau}_g \} \text{ for } y \in \mathcal{G}_g
\end{align*}
where $g \in \{h,t\}$ denotes the group that $y$ belongs to.
This method provides a partition-conditional coverage guarantee,
\begin{lemma}[Partition-conditional Coverage Gaurantee]
\label{PW_Cov}
The prediction sets $\mathcal{C}=\mathcal{C}_{\text{PW}}$ from the Partition-Wise method achieve partition-conditional coverage:
\begin{align*}
    &\mathbb{P}(Y_{n+1}\in\mathcal{C}(X_{n+1})| Y_{n+1}\in\mathcal{G}_h)\geq 1-\alpha,\\
    &\mathbb{P}(Y_{n+1}\in\mathcal{C}(X_{n+1})| Y_{n+1}\in\mathcal{G}_t)\geq 1-\alpha
\end{align*}
where $\alpha$ is the pre-defined miscoverage level.
\end{lemma}
The proof of Lemma~\ref{PW_Cov} follows the similarity to that of the CLASSWISE method \cite{sadinle2019least,DBLP:journals/ml/Vovk13}. 
However, previous studies \cite{DBLP:conf/nips/DingABJT23} have shown that when $n_h$ or $n_t$ is small, the resulting coverage estimations can exhibit high variance and deviate significantly from the target level $1-\alpha$.

\subsection{Definitions of Non-conformity Scores}

\begin{itemize}
    \item APS: The Adaptive Prediction Sets (APS) method proposed by \citet{romano2019conformalized} aims to approximate X-conditional coverage and is defined as follows,
        \begin{align}
        \label{APS_Score}
        s_{\text{APS}}(\bm{x},y)=
        \sum_{i=1}^{o_{\bm{x}}(y)-1}\hat{\pi}_{(i)}(\bm{x})+u\cdot\hat{\pi}_{y}(\bm{x}),
    \end{align}
    where $o_{\bm{x}}(y)$ is the rank of label $y$ under the estimated posterior class probabilities $\hat{\pi}(\bm{x})$ and $u\sim U[0,1]$ is a random tie-breaking variable.
    \item LAC: The Least Ambiguous Set-valued Classifier (LAC) method \cite{sadinle2019least} is  theoretically guaranteed to produce the smallest valid prediction set under the assumption of accurate class posterior estimates,
    \begin{align*}
        s_{\text{LAC}}(\bm{x},y)=1-\hat{\pi}_{y}(\bm{x}).
    \end{align*}
    \item TOPK: The TOPK method, introduced in \citet{RAPS2021} constructs a fixed-size prediction set by selecting the top-$k$ labels with the highest estimated probabilities, where $k$ is chosen to meet the target coverage level,
       \begin{align}
       \label{TOPK-score}
        s_{\text{TOPK}}(\bm{x},y)=o_{\bm{x}}(y)+u,
    \end{align}
    where $o_{\bm{x}}(y)$ is the rank of label $y$ in descending order of predicted probabilities $\hat{\pi}(\bm{x})$ and $u\sim U[0,1]$ is a uniform random variable for tie-breaking.
    \item RAPS: To improve the efficiency of APS, \citet{RAPS2021} introduced the regularization APS method, which adds a penalty for including low-ranked labels, 
            \begin{align*}
        s_{\text{RAPS}}(\bm{x},y)=
        \sum_{i=1}^{o_{\bm{x}}(y)-1} \hat{\pi}_{(i)}(\bm{x})+u\cdot\hat{\pi}_{y}(\bm{x})+\lambda_{\text{RAPS}}(o_{\bm{x}}(y)-k_{\text{RAPS}})^{+},
    \end{align*}
    where $\lambda_{\text{RAPS}}\in \mathbb{R}^+$, $k_{\text{RAPS}}\in\mathbb{N}$ are hyperparameters, and $(\cdot)^+$ denotes the ReLU function.
    For all methods compared in Section~\ref{section_headtail_exp}, includuing STANDARD, PW, and TACP, 
    we set $k_{\text{RAPS}}=8$ and $\lambda_{\text{RAPS}}=0.01$ for ImageNet LT, and $k_{\text{RAPS}}=5$ with $\lambda_{\text{RAPS}}=0.01$ on CIFAR100 LT.
\end{itemize}

\subsection{Parameters Finetuning of TACP}
While several values of $k_r$ and $\lambda$ can reduce the coverage gap, certain values will yield smaller head-tail coverage gaps. 
In our experiments, $\lambda$ and $k_r$ are chosen using the calibration dataset.
We select $k_r$ from a candidate set $\{1,2,\dots,10\}$ for CIFRAR100 LT and from $\{1,2,\dots,15\}$ for ImageNet LT. 
For APS and LAC non-conformity scores, $\lambda$ is chosen from $\{0.001,0.01,0.02,0.05,0.1,0.5,1\}$. In the case of RAPS score, we set $(\lambda_{\text{reg}},k_{\text{reg}})=(0.01,5)$ for CIAR100 LT  and $(\lambda_{\text{reg}},k_{\text{reg}})=(0.01,8)$ for ImageNet LT, using the same $\lambda$ candidate values for APS and LAC.
Since TOPK non-conformity score Eq.~\eqref{TOPK-score} already incorporates ranking information and tends to be much larger in magnitude than a softmax‐based score, we choose $\lambda$ from $\{1,2,3,4,5\}$ instead.

\subsection{Results of Varying Head-Tail Partition}
We analyze the STANDARD, PW, and TACP methods in partition $\eta =50\%$ on CIFAR100 LT and ImageNet LT in Table~\ref{CovGap_HT_table}. We also compare STADNARD and TACP methods at $\eta = 70\%$ on CIFAR100 LT in Figure~\ref{fig:head_tail_coverage}.
To further evaluate the robustness of our method, we provide additional experimental results with a wider range of candidate values for the partition parameter $\eta$.

\paragraph{Fixed $\lambda$ and $k_r$} 
We evaluate the head- and tail-conditional coverages at $\alpha=10\%$ for both the STANDARD and TACP methods using four non-conformity scores on CIFAR100 LT ($\mu=100$) and ImageNet LT ($\rho=0.3$), under varying partition values $\eta$ from $10\%$ to $60\%$ in steps of $10\%$. For the TACP method, we use a fixed pair of $(k_r,\lambda)$ across all values of $\eta$.
Figure~\ref{fixed_kr_covgap} shows that the TACP method outperforms the STANDARD across all scores, illustrated by the reduced gap between head and tail curves.
Overall, our TACP method achieves a smaller coverage gap between head and tail, even with fixed $\lambda$ and $k_r$ for each value of $\eta$.

\paragraph{Partition-driven $\lambda$ and $k_r$}
To compare the performance of STANDARD, Partition-Wise, and TACP methods, we further conduct experiments on ImageNet LT ($\rho=0.3$) using APS, LAC, and TOPK non-conformity scores under varying partition values $\eta$ from $10\%$ to $60\%$ in steps of $10\%$. For TACP, we select $\lambda$ and $k_r$ in a partition-driven manner.

Figure~\ref{varying_kr_covgap} illustrates that both TACP and PW reduce the coverage gap compared with STA methods across all three non-conformity scores. For example, at $\eta=40\%$, the CovGap of STA-APS is about $3\%$, whereas PW-APS reduces it to approximately $1.4\%$, and TACP-PAS achieves an even smaller value around $0.6\%$.
Moreover, TACP is consistently comparable to or outperforms PW for each partition and score, as it adaptively adjusts its penalty strength based on the chosen head-tail partition, thereby minimizing the head-tail conditional coverage gap for each $\eta$.

\begin{figure*}
    \centering
    \begin{subfigure}[t]{0.5\textwidth}
        \centering
        \includegraphics[width=\linewidth]{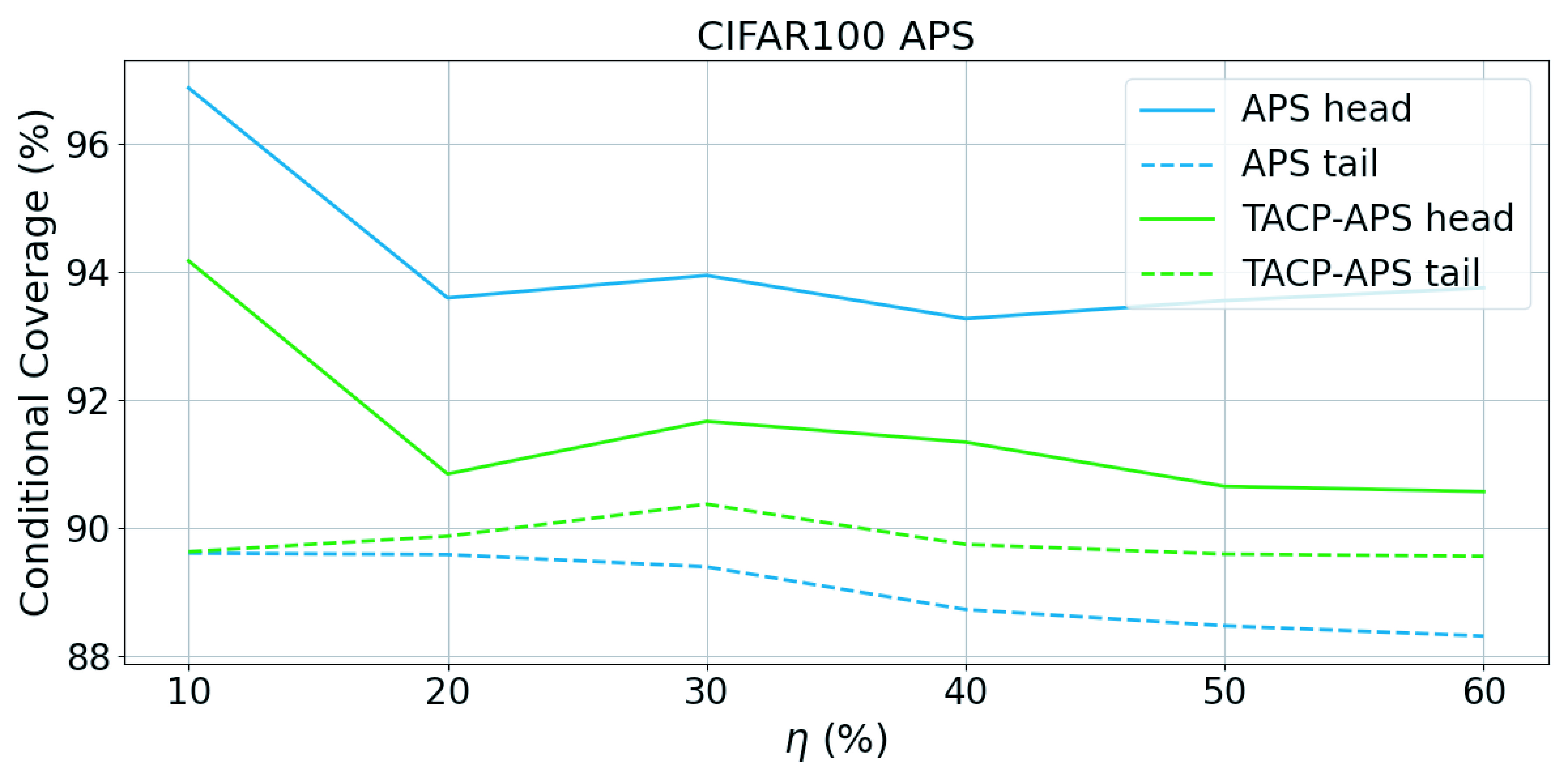}
    \end{subfigure}%
    \hfill
    \begin{subfigure}[t]{0.5\textwidth}
        \centering
        \includegraphics[width=\linewidth]{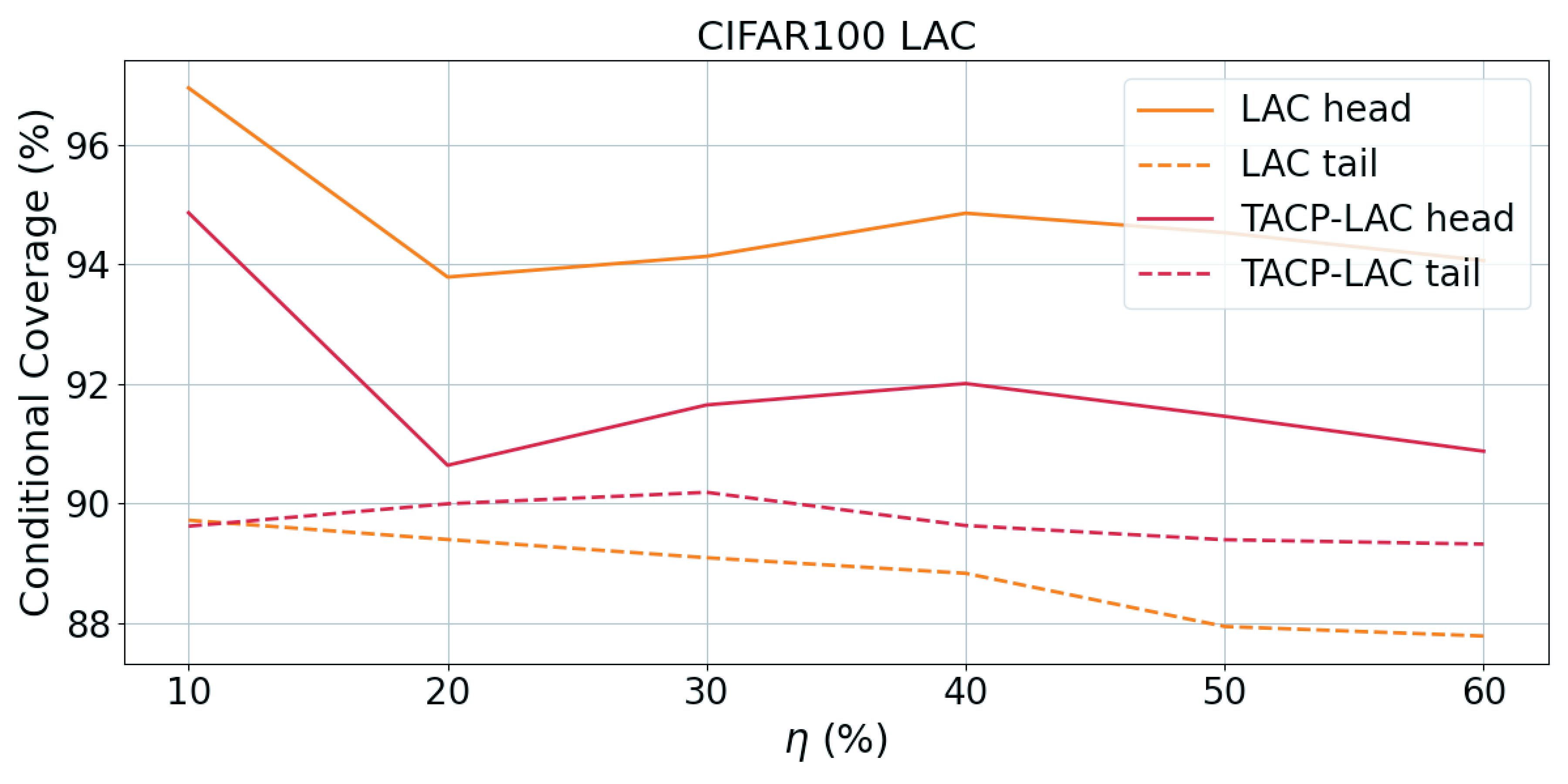}
    \end{subfigure}
        \begin{subfigure}[t]{0.5\textwidth}
        \centering
        \includegraphics[width=\linewidth]{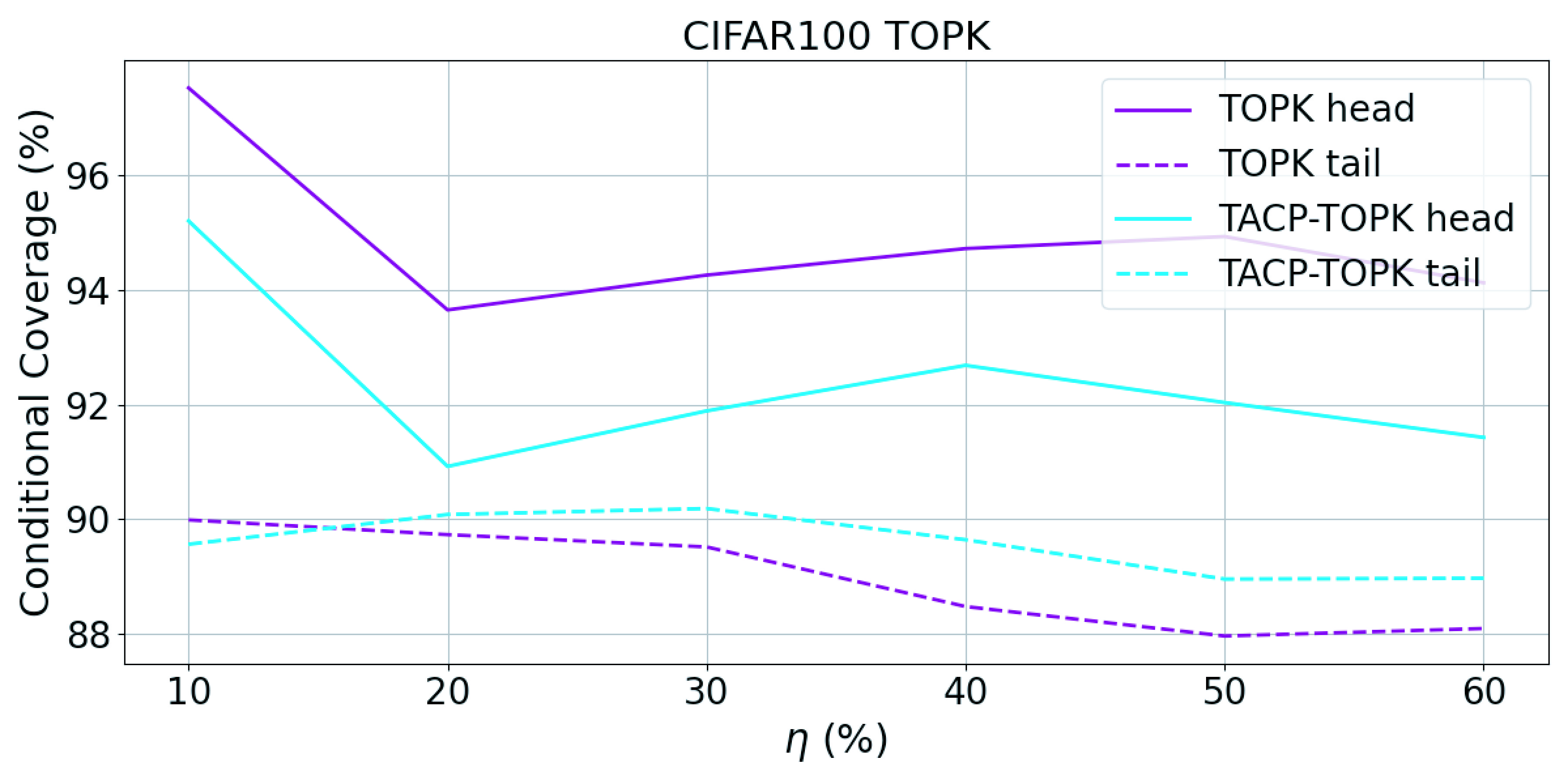}
    \end{subfigure}%
    \hfill
    \begin{subfigure}[t]{0.5\textwidth}
        \centering
        \includegraphics[width=\linewidth]{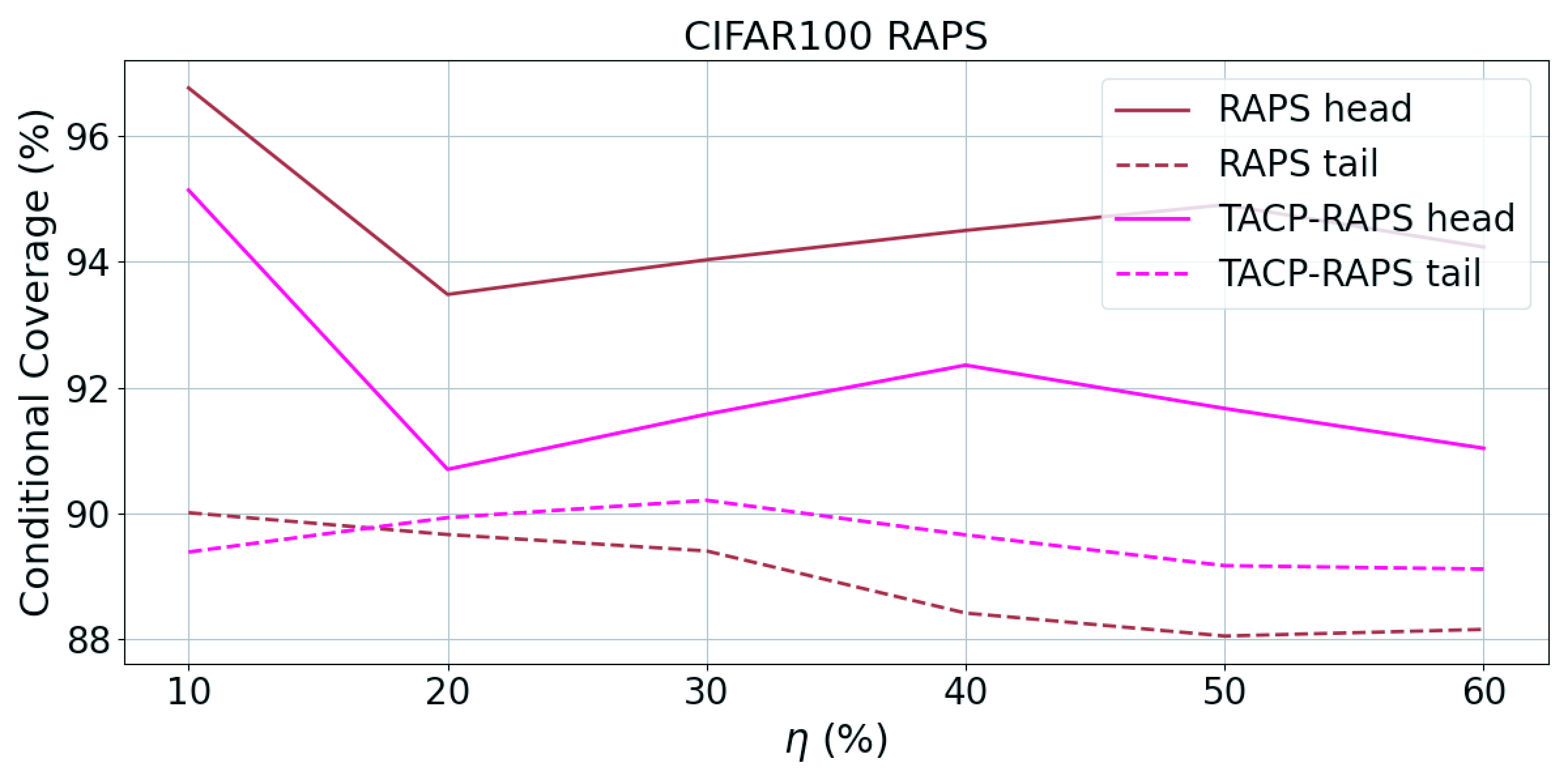}
    \end{subfigure}
    \begin{subfigure}[t]{0.5\textwidth}
        \centering
        \includegraphics[width=\linewidth]{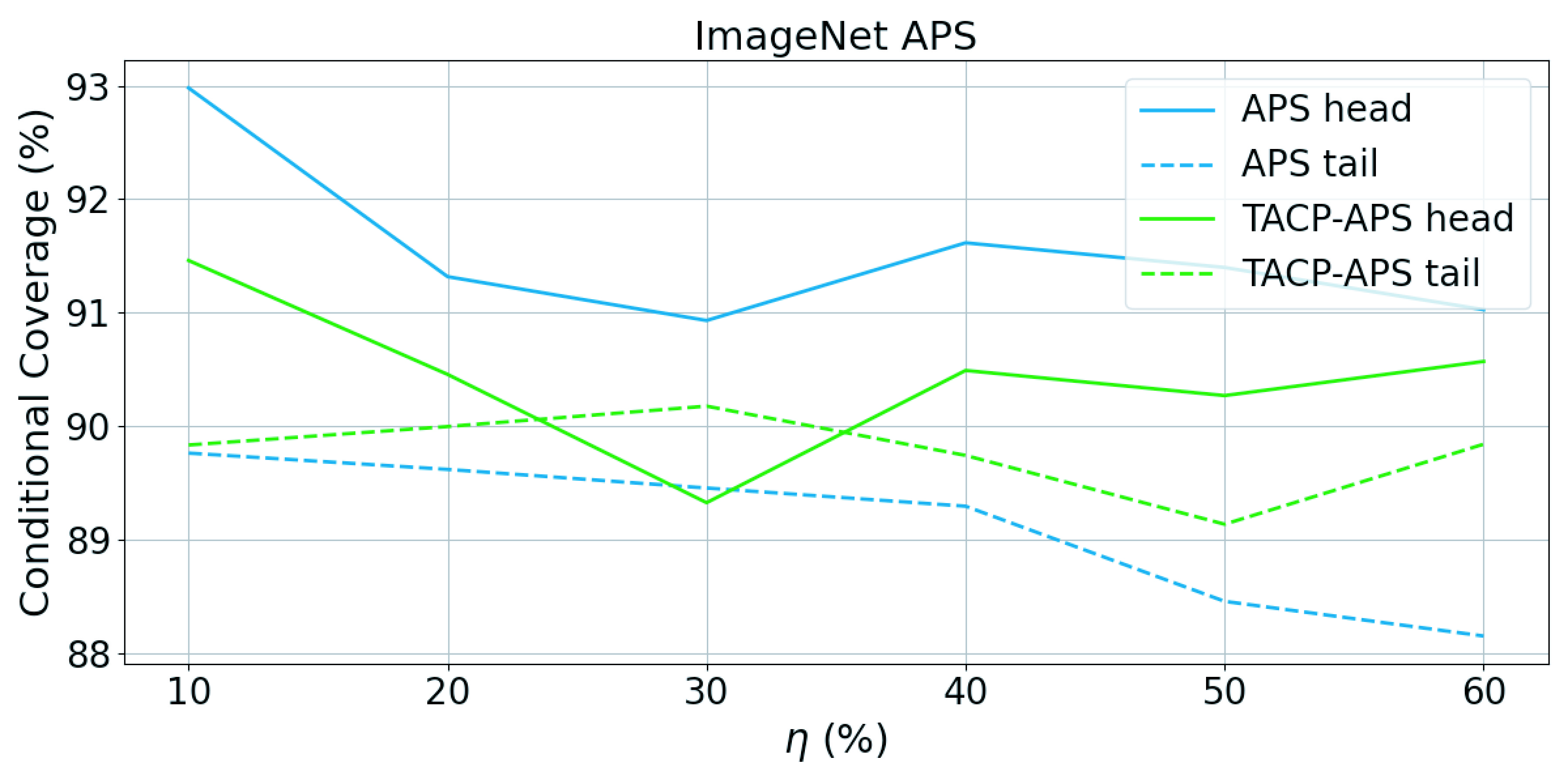}
    \end{subfigure}%
    \hfill
    \begin{subfigure}[t]{0.5\textwidth}
        \centering
        \includegraphics[width=\linewidth]{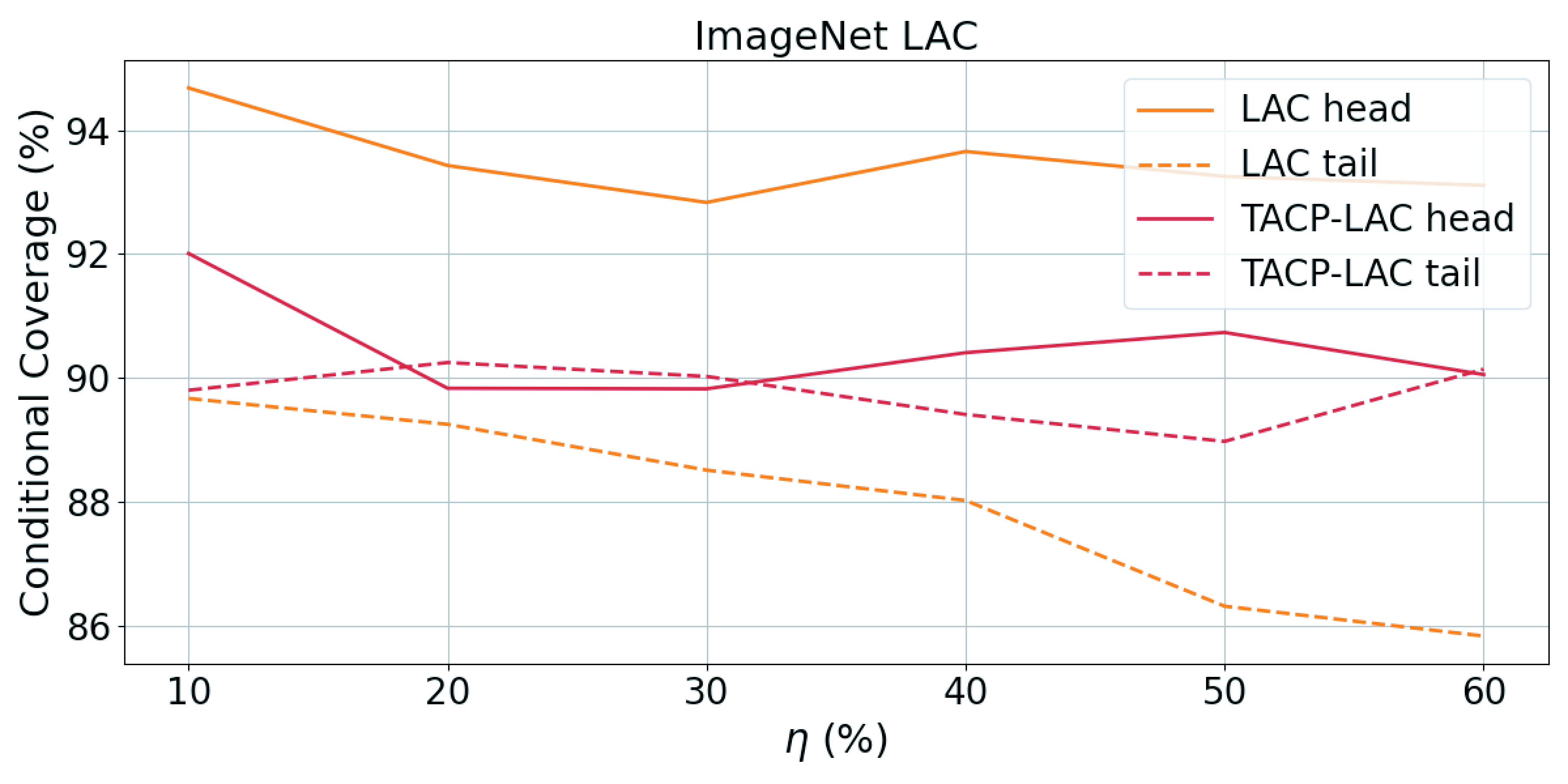}
    \end{subfigure}
            \begin{subfigure}[t]{0.5\textwidth}
        \centering
        \includegraphics[width=\linewidth]{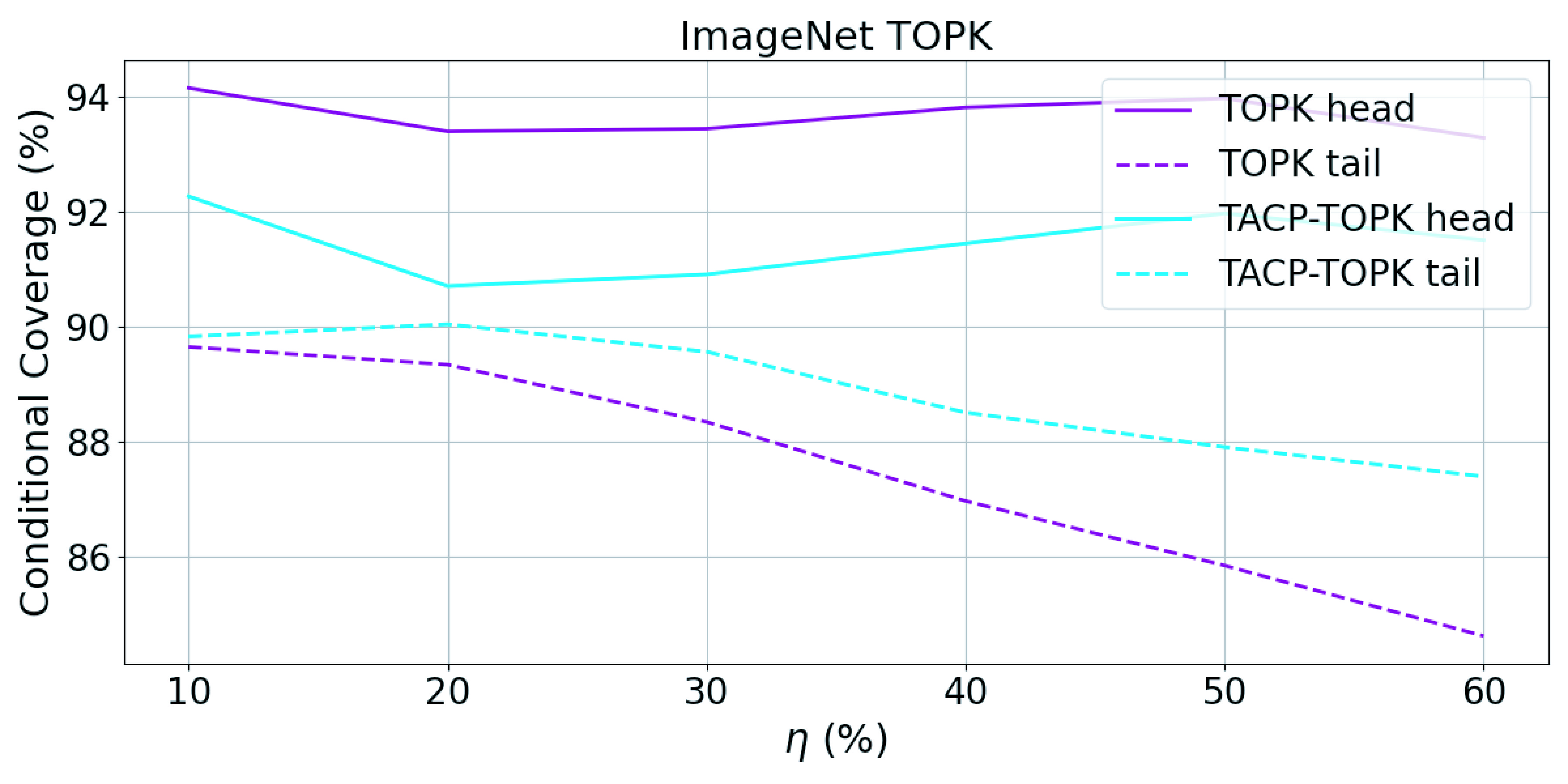}
    \end{subfigure}%
    \hfill
    \begin{subfigure}[t]{0.5\textwidth}
        \centering
        \includegraphics[width=\linewidth]{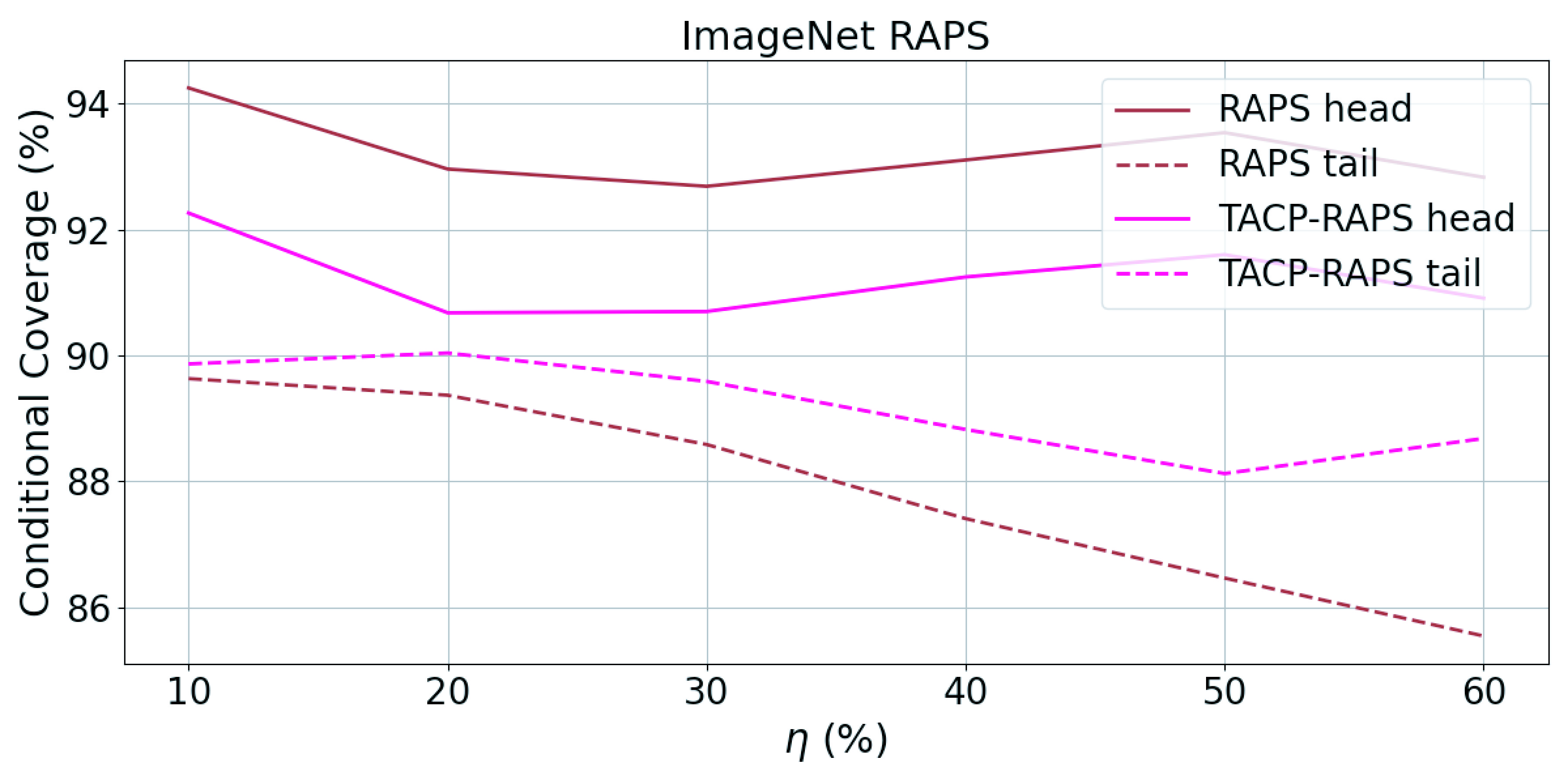}
    \end{subfigure}
    \caption{ Head- and tail-conditional coverage under varying head-tail partition for STANDARD and TACP (fixed $k_r$ and $\lambda$) methods using four non-conformity scores on CIAFR100 LT ($\mu=100$) and ImageNet LT ($\rho=0.3$) at $\alpha=10\%$.}
    \label{fixed_kr_covgap}
\end{figure*}

\begin{figure*}
    \centering
    \begin{subfigure}[t]{0.32\textwidth}
        \centering
        \includegraphics[width=\linewidth]{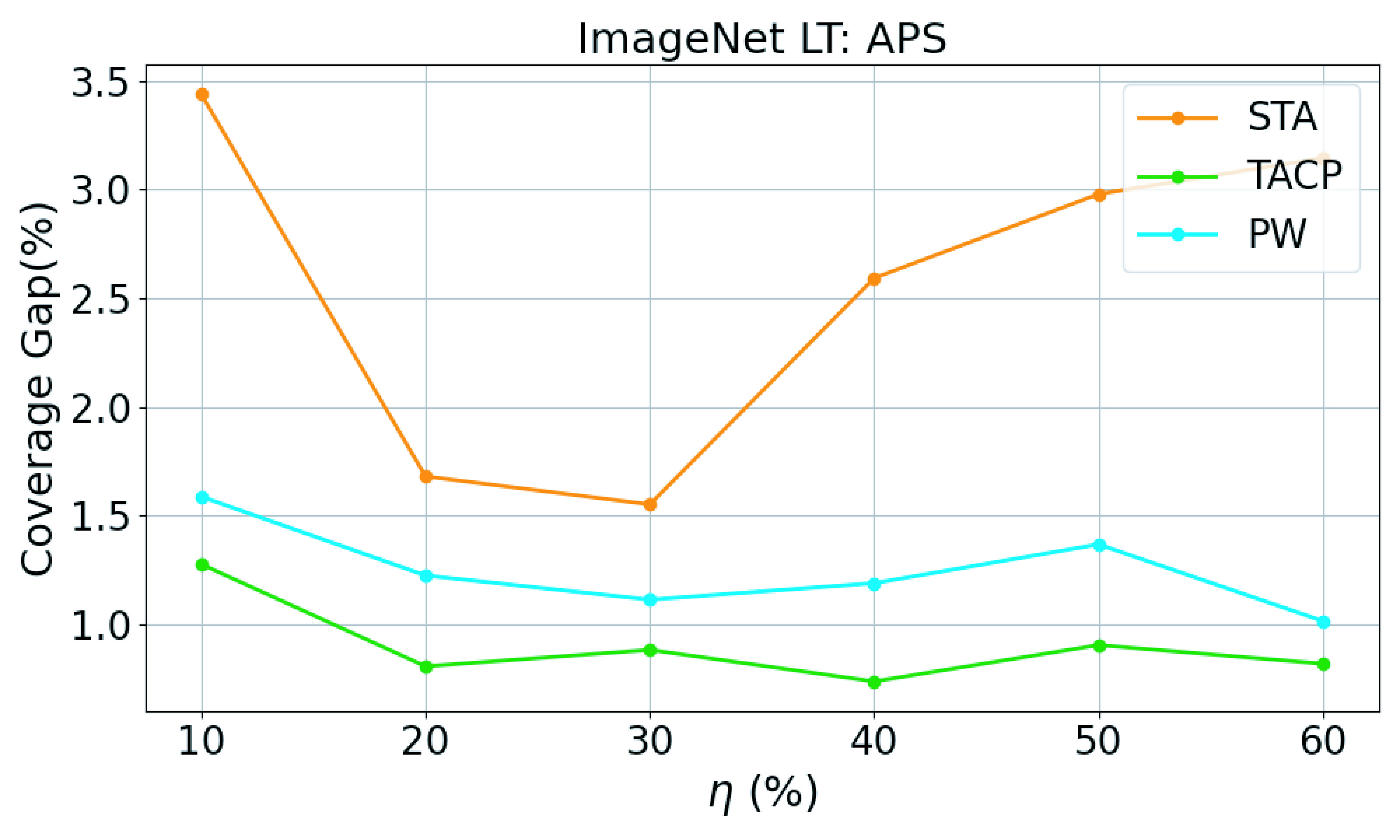}
    \end{subfigure}
    \begin{subfigure}[t]{0.32\textwidth}
        \centering
        \includegraphics[width=\linewidth]{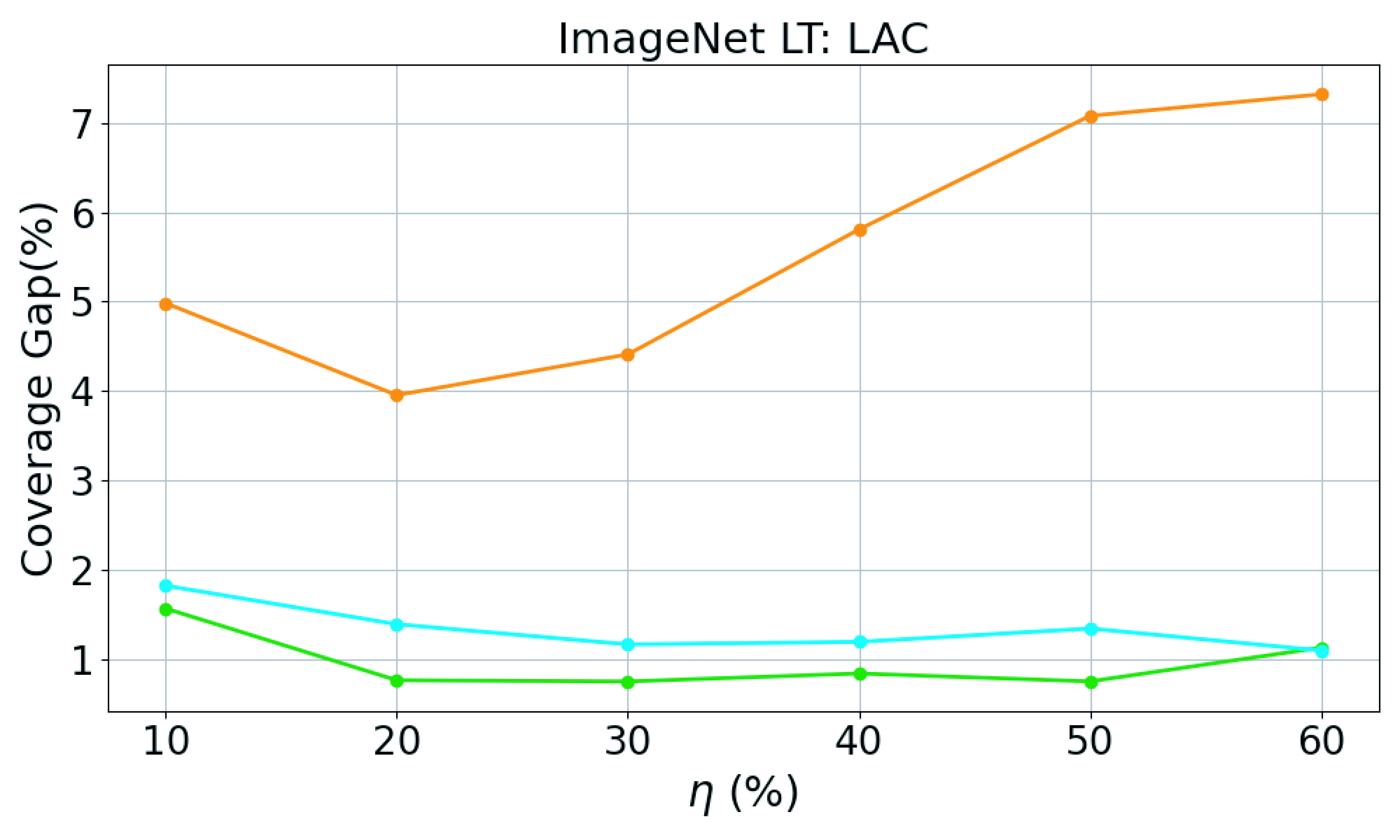}
    \end{subfigure}
        \begin{subfigure}[t]{0.32\textwidth}
        \centering
        \includegraphics[width=\linewidth]{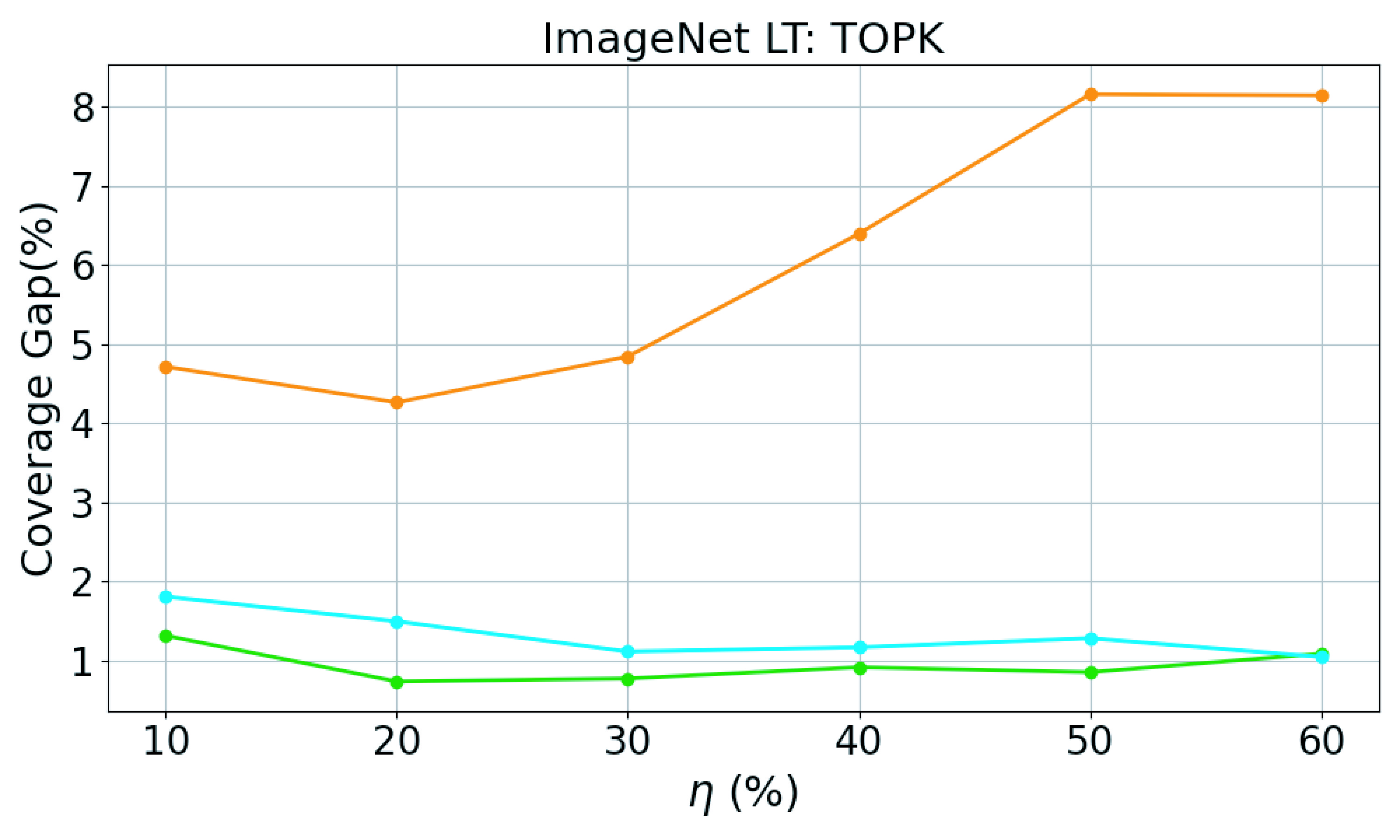}
    \end{subfigure}
    \caption{CovGap-HT for STANDARD, Partition-Wise, and TACP (data driven $k_r$ and $\lambda$) methods on ImageNet LT ($\rho=0.3$) under varying head–tail partition $\eta\%$ at $\alpha=10\%$.}
    \label{varying_kr_covgap}
\end{figure*}

\section{Proofs}
\label{Appendix_proof}
For simplicity, we prove the result for $\lambda=1$, 
and the derivations for other values of $\lambda\in \mathbb{R}^+$ are similar.

\begin{proof}[Proof of Theorem~\ref{marginal_cov}]
Let $\{(X_i,Y_i)\}_{i=1}^n$ be the $n$ calibration samples, and $(X_{n+1},Y_{n+1})$ be the test sample. Define the TACP method with $s(\bm{x},y)$ as base score as follows,
\begin{align*}
    s_{\text{TACP}}(\bm{x},y)=s(\bm{x},y)+\mathbb{I}(y\in \mathcal{G}_h)(o_{\bm{x}}(y)-k_r)^{+},
\end{align*}
where $\mathbb{I}$ is the indicator function; $o_{\bm{x}}(y)$ denotes the predicted rank of label $y$ for feature $\bm{x}$; and $k_r$ is a fixed positive iterger number.

For each calibration sample $(X_i,Y_i)$, we compute its non-conformity score $s_i := s_{\text{TACP}}(X_i, Y_i)$. Then the conformal prediction set for $X_{n+1}$ is defined as:
\begin{align*}
    \mathcal{C}_{\text{TACP}}(X_{n+1})=\{y: s_{\text{TACP}}(X_{n+1},y)\leq \hat{q}_\alpha\}
\end{align*}
where 
\begin{align*}
    \hat{q}_\alpha := \lceil(1-\alpha)(n+1)\rceil\text{-th smallest value in }\{s_1,s_2,\dots,\infty\}.
\end{align*}
We first show the lower bound:
\begin{align*}
    \mathbb{P}(Y_{n+1}\in \mathcal{C}_{\text{TACP}}(X_{n+1}))\geq 1-\alpha.
\end{align*}
Following the i.i.d. of $\{(X_i,Y_i)\}_{i=1}^n$ and $(X_{n+1},Y_{n+1})$, the rank of $s_{n+1}$ among $(s_1,\dots,s_n,s_{n+1}$ is uniformly distributed on $(1,\dots,n+1)$, thus
\begin{align*}
    \mathbb{P}(s_{n+1}\leq \hat{q}_\alpha)\geq 1-\alpha
\end{align*}

For the upper bound, we assume that ties among scores are broken using the uniform random variables $u$, so that $s_1,\dots,s_n,s_{n+1}$ are almost surely distinct. 
\begin{align*}
    s_{\text{TACP}}(\bm{x},y)=s(\bm{x},y)+\mathbb{I}(y\in \mathcal{G}_h)(o_{\bm{x}}(y)-k_r+u)^{+},
\end{align*}
where $u\sim U(0,1)$ is an independent random variable for tie-breaking. Then the probability that $s_{n+1}$ ranks among the smallest $\lceil(1-\alpha)(n+1)\rceil$ is exactly
\begin{align*}
    \mathbb{P}(s_{n+1}\leq \hat{q}_\alpha) =\frac{\lceil(1-\alpha)(n+1)\rceil}{n+1}\leq (1-\alpha)+\frac{1}{n+1}.
\end{align*}
This completes the proof of both bounds.
\end{proof}

\begin{lemma}[Acceptance Threshold Comparison]
Under the same calibration dataset and non-conformity score, the $1-\alpha$ quantile $\hat{q}(k_r)$ computed by TACP is greater than or equal to $\hat{q}_0$ from the STANDARD method for any $k_r$.
\label{qcompare}
\end{lemma}

\begin{proof}[Proof of Lemma~\ref{qcompare}]
    Let $E_{xy}$ be the event that the calibration dataset is fixed as $\{(X_i,Y_i)\}_{i=1}^n=\{(\bm{x}_1,y_1),\dots, (\bm{x}_n,y_n)\}$. Under this event, define $\hat{q}$ and $\hat{q}_0$ as the $\frac{\lceil(1+n)(1-\alpha)\rceil}{n}$ quantiles computed by TACP and STANDARD method, respectively, using the same base score, denoted by $s_0(\bm{x}, y)$. Then we have,
\begin{align*}
    &\text{For any } i\in [n], k_r\in \mathbb{N}, ~ s_0(\bm{x}_i,y_i) \leq  s_0(\bm{x}_i,y_i) +\mathbb{I}(y_i \in \mathcal{G}_h)\cdot(o_{\bm{x}_i}(y_i)-k_r)^{+} \\
\Longrightarrow &\text{Quantile}(1-\alpha,\{s_0(\bm{x}_i,y_i)\}_{i=1}^n) \leq \text{Quantile}(1-\alpha,\{s_0(\bm{x}_i,y_i)+\mathbb{I}(y_i \in \mathcal{G}_h)\cdot(o_{\bm{x}_i}(y_i)-k_r)^{+}\}_{i=1}^n)\\
\Longrightarrow   &\hat{q}_{0} \leq  \hat{q}(k_r)
\end{align*}
\end{proof}
Lemma~\ref{qcompare} holds because TACP applies rank-based penalties selectively to head labels $\mathcal{G}_h$, aiming to reduce the head–tail coverage gap.

\begin{lemma}[Existence of $k_r$]
For any test sample pair $(X_{n+1},Y_{n+1})\sim P$, there exists $k_r \in \mathbb{N}$, such that the following inequality holds:
\begin{align*}
P(s(X_{n+1},Y_{n+1}) \leq \hat{q}_0 \mid  Y_{n+1} \in \mathcal{G}_h,E_{xy})
 \geq P( s_{\text{TACP}}^{k_r}(X_{n+1},Y_{n+1}) \leq \hat{q}(k_r)\mid  Y_{n+1} \in \mathcal{G}_h,E_{xy})
\end{align*}
where 
\begin{align*}
s_{\text{TACP}}^{k_r}(X_{n+1},Y_{n+1}) := s(X_{n+1},Y_{n+1})+(o_{X_{n+1}}(Y_{n+1})-k_r)^{+}.    
\end{align*}
\label{existkr}
\end{lemma}
\begin{proof}[Proof of Lemma~\ref{existkr}]
Let $f(k_r)$ denote the head-conditional coverage of TACP under fixed calibration dataset $E_{xy}$,
\begin{align*}
    f(k_r) := P( s_{\text{TACP}}^{k_r}(X_{n+1},Y_{n+1}) \leq \hat{q}(k_r) \mid  Y_{n+1} \in \mathcal{G}_h,E_{xy}).
\end{align*}
The rank term $o_{X}(Y)$ takes values in $\{1,2,\dots, K\}$, where $K=|\mathcal{Y}|$. Thus, for any $k_r \geq  K$, we have 
\begin{align*}
    &(o_{X_{n+1}}(Y_{n+1})-k_r)^{+}=0 \text{ almost surely}\\
    \Longrightarrow & s_{\text{TACP}}^{k_r}(X_{n+1},Y_{n+1}) = s(X_{n+1},Y_{n+1})   \text{ almost surely } \\
    \Longrightarrow  &\hat{q}(k_r) =\hat{q}_0 \\
    \Longrightarrow  &f(k_r)= P(s(X_{n+1},Y_{n+1}) \leq \hat{q}_0 \mid  Y_{n+1} \in \mathcal{G}_h,E_{xy})
\end{align*}
Thus, there exists $k_r$ such that 
\begin{align*}
    f(k_r) \leq  P(s(X_{n+1},Y_{n+1}) \leq \hat{q}_0 \mid  Y_{n+1} \in \mathcal{G}_h,E_{xy}).
\end{align*}
\end{proof}

Lemma~\ref{existkr} ensures that the TACP coverage on the head group can always be controlled by properly tuning the regularization parameter $k_r$. In practice, we search over $k_r$ to minimize the head–tail coverage gap while ensuring valid coverage.
Combining Lemma~\ref{qcompare} and Lemma~\ref{existkr}, we can obtain the following improved coverage gap result.

\begin{proof}[Proof of Theorem~\ref{theorem2}]
Let $P$ denote the joint distribution over $(X_{n+1},Y_{n+1})$, then we have
\begin{align*}
   &P(Y_{n+1}\in \mathcal{C}_{\text{TACP}}(X_{n+1}) \mid  Y_{n+1}\in \mathcal{G}_h,E_{xy})
  -   P (Y_{n+1}\in \mathcal{C}_{\text{TACP}}(X_{n+1})\mid  Y_{n+1}\in \mathcal{G}_t,E_{xy})\\
=&P(s(X_{n+1},Y_{n+1})
+(o_{X_{n+1}}(Y_{n+1})-k_r)^{+} \leq \hat{q}(k_r) \mid  Y_{n+1} \in \mathcal{G}_h,E_{xy}) 
- P(s(X_{n+1},Y_{n+1})\leq \hat{q}(k_r) \mid Y_{n+1} \in \mathcal{G}_t,E_{xy})\\
\text{\textcircled{1}}~ \leq& P(s(X_{n+1},Y_{n+1})
+(o_{X_{n+1}}(Y_{n+1})-k_r)^{+} \leq \hat{q}(k_r)  \mid  Y_{n+1} \in \mathcal{G}_h,E_{xy}) 
- P(s(X_{n+1},Y_{n+1})\leq \hat{q}_0 \mid Y_{n+1} \in \mathcal{G}_t,E_{xy})\\
\text{\textcircled{2}}~ \leq& P(s(X_{n+1},Y_{n+1}) \leq \hat{q}_0 \mid  Y_{n+1} \in \mathcal{G}_h,E_{xy}) - P(s(X_{n+1},Y_{n+1})\leq \hat{q}_0 \mid Y_{n+1} \in \mathcal{G}_t,E_{xy})\\
= &P(Y_{n+1}\in \mathcal{C}_{\text{STD}}(X_{n+1}) \mid  Y_{n+1}\in \mathcal{G}_h,E_{xy})
- P(Y_{n+1}\in \mathcal{C}_{\text{STD}}(X_{n+1})\mid  Y_{n+1}\in \mathcal{G}_t,E_{xy})
\end{align*}
The inequality \textcircled{1} follows from Lemma ~\ref{qcompare}, while the inequality \textcircled{2} holds by Lemma ~\ref{existkr}.
\end{proof}

\section{Details of Class-conditional Experiments}
\label{Appendix_classwise-exp}

\subsection{Baselines}
\paragraph{CLASSWISE Method}
The CLASSWISE (CW) method computes a separate acceptance threshold for each class based on class-specific calibration samples. And the prediction set is defined as
\begin{align*}
    \mathcal{C}_{\text{CW}}(X_{n+1})=\{y: s(X_{n+1},y)\leq \hat{\tau}_{\alpha}^y\}, \text{ for all } y\in \mathcal{Y},
\end{align*}
where the threshold $\hat{\tau}_{\alpha}^y$ for label $y$ is given by
\begin{align*}
    \hat{\tau}_{\alpha}^y = \text{Quantile}\left(\frac{\lceil(1+n_y)(1-\alpha)\rceil}{n_y},\{s(X_i,Y_i)\}_{i\in \mathcal{I}_y} \right),
\end{align*}
and $n_y$ represents the number of calibration samples belonging to label $y$ and $\mathcal{I}_y=\{i\in [n]: Y_i=y\}$ denotes the indices of these samples. Note that we set $\hat{\tau}_{\alpha}^y=\infty$ when $n_y<(1/\alpha)-1$, making the prediction set for features with rare classes the entire label space $\mathcal{Y}$.
Although each ground-truth class $y$ has at least $1-\alpha$ probability of being included in the prediction set \cite{DBLP:journals/ml/Vovk13}, 
it often suffers from overly large prediction sets for classes with limited calibration samples.

\paragraph{CLUSTER Method}
The CLUTSER (CLUS) method groups classes into $M$ clusters based on their non-conformity scores on a cluster dataset and computes the acceptance threshold for each cluster.
\begin{align*}
    \mathcal{C}_{\text{CLUS}}(X_{n+1})=\{y: s(X_{n+1},y)\leq \hat{\tau}_{\alpha}^{h(y)}\},
\end{align*}
for all $h(y) \in \{1,\dots,M\}\cup\{\text{null}\}$, where 
\begin{align*}
    &\hat{\tau}_{\alpha}^m = \text{Quantile}\left(\frac{\lceil(1+n_m)(1-\alpha)\rceil}{n_m},\{s(X_i,Y_i)\}_{i\in \mathcal{I}_m} \right),\\
        &\hat{\tau}_{\alpha}^{\text{null}} = \text{Quantile}\left(\frac{\lceil(1+n)(1-\alpha)\rceil}{n},\{s(X_i,Y_i)\}_{i\in \mathcal{I}} \right).
\end{align*}
where $\mathcal{I}_m=\{i\in [n']: h(Y_i)=m\}$ and $n_m = |\mathcal{I}_m|$ represents the indices and count of samples in cluster m, while $\mathcal{I}=\cup_m \mathcal{I}_m$ and $n=\sum_m n_m$ aggregate all clusters. Classes with too few data to join any of the clusters form the ‘null’ set, and we aggregate all calibrated samples to calculate the acceptance threshold for this ‘null’ group. 
Following \citet{DBLP:conf/nips/DingABJT23}, we set $\gamma=0.8$ and $M=4$, where $\gamma\in [0,1]$ denotes the fraction of calibration data used for clustering and $M$ specifies the number of clusters.

\paragraph{RC3P Method}
By introducing label rank calibration and adapting class-wise quantile thresholds, the Rank Calibrated Class-conditional Conformal Prediction (RC3P) method improves the prediction efficiency of CLASSWISE by leveraging the class-wise top-$k$ prediction error, denoted by 
\begin{align}
    \text{Err}_y^{k}=\mathbb{E}_{X,Y}[\mathbb{I}(Y\notin \text{Top-}k(X))\mid Y=y ],
\end{align}
to adjust the quantile level for each class accordingly. The prediction set is defined as
\begin{align*}
    \hat{\mathcal{C}}_{\text{RC3P}}(X_{n+1})=
    \{y: s(X_{n+1},y)\leq \hat{\tau}(y), o_{X_{n+1}}(y)\leq \hat{k}(y) \}
\end{align*}
where 
$$\hat{\tau}(y)=\text{Quantile}(1-(\alpha-\text{Err}_y^{\hat{k}(y)}),\{(s(X_i,Y_i))\}_{i\in \mathcal{I}_y}).$$
In \citet{ShiGBD024}, they proposes two variants: RC3P-I (model-agnostic coverage), where $\hat{k}(y)\in \{k: \text{Err}_y^{k}<\alpha\}$ and RC3P-II (model-agnostic coverage with improved predictive efficiency), where $\hat{k}(y)=\min\{k: \text{Err}_y^{k}<\alpha\}$. They evaluated RC3P under a partial LT setting, where only the training dataset is imbalanced, while the calibration and test datasets remain balanced. 
In our fully LT setting, we find that RC3P-II often leads to invalid marginal coverage, likely due to unreliable estimations of the class-wise top-$k$ error when all the training, calibration, and test datasets are long-tail. Therefore, we report the results of RC3P I in our experimental results

\subsection{Evaluation Metric}
Denote the test dataset by $\{(\bm{x}'_i,y'_i)\}_{i=1}^{N_{\text{test}}}$. We measure the prediction sets by coverage and average set size (AvgSize):
\begin{align*}
&\text{Coverage}=\frac{1}{N_{\text{test}}}\sum _{i=1}^{N_{\text{test}}}\mathbb{I}(y\in \mathcal{C}(\bm{x}_{i}) ), ~\text{AvgSize} =\frac{1}{N_{\text{test}}}\sum _{i=1}^{N_{\text{test}}}|\mathcal{C}(\bm{x}_{i})|,
\end{align*}
where $\mathbb{I}$ is the indicator function, $N'$ denotes the number of test samples, and $|\mathcal{C}(\bm{x}_{i})|$ denotes the number of labels in the prediction set $\mathcal{C}(\bm{x}_{i})$.
Let $\mathcal{I}_y=\{i \in [N_{\text{test}}]: y'_i = y\}$ be the indices of test examples with label y, we define the class-conditional coverage gap (CovGap) as follows,
\begin{align}
     \text{CovGap}= 100\times
   \frac{1}{|\mathcal{Y}|} \sum_{y\in \mathcal{Y}} 
   \bigg | \frac{\sum_{j\in \mathcal{I}_y}\mathbb{I}(y_j\in \mathcal{C}(\bm{x}_j))}{|\mathcal{I}_y|} -(1-\alpha) \bigg |,
\end{align}
where $|\mathcal{I}_y|$ denotes the number of test samples associated with label $y$, and $\alpha$ is the predefined miscoverage level. 
CovGap quantifies the adaptiveness of conformal prediction sets across labels by measuring the derivation of class-conditional coverage from the target level $1-\alpha$.
A smaller CovGap indicates better adaptiveness, and CovGap = 0 when the prediction sets achieve feature-conditional coverage.

\subsection{Parameters Finetuning of sTACP}
We select $k_r$ in a similar way to the TACP method on ImageNet LT. 
For APS and LAC non-conformity scores, $\lambda$ is chosen from $\{0.001,0.01,0.1,0.2,0.3,0.5,0.8,1\}$.
and for RAPS and TOPK from $\{1,8,10,50,100,300,500,800\}$ instead.

\subsection{Results at Target Coverage level 95\%} 
We further conduct class-conditional experiments at $\alpha=5\%$, using LAC as base score to provide a comparison of the baseline performances.

Table~\ref{ImageNetClasswise_inefficency} shows that under the long-tail label distribution with limited data, the CLASSWISE method often generates the entire label set $\mathcal{Y}$ for almost all input samples, resulting in a trivial average coverage of $100\%$ and a class-conditional coverage gap of $\alpha\times100=5$. Such unnecessarily larger label prediction sets, with size close to $|\mathcal{Y}|=1000$, are uninformative and ineffective.
The CLUSTER method degenerates to the STANDARD method since each class in the cluster dataset contains fewer than $1/\alpha-1=19$ samples, causing all classes to be assigned to the ‘null’ cluster, so we omit its results. Moreover, the RC3P method suffers from unreliable class-specific top-$k$ error estimations, which leads to noisy class-wise acceptance threshold estimations.
In contrast, our sTACP method achieves a smaller CovGap while preserving prediction efficiency. 

\begin{table}[H]
  \centering
    \fontsize{9pt}{12pt}\selectfont
  \begin{tabular}{cccc}
   \toprule
    Method & CovGap & Coverage & AvgSize \\
    \midrule
    STA&9.64$\pm$0.00&95.12$\pm$0.01&33.44$\pm$2.51\\
    CW&5.00$\pm$0.00&99.90$\pm$0.11&983.15$\pm$3.21\\
    CLUSTER&{———}&{———}&{———}\\
    RC3P&15.48$\pm$0.67&89.40$\pm$0.66&32.27$\pm$2.34\\
    sTACP&8.12$\pm$0.03&0.958$\pm$0.08&48.31$\pm$0.72\\
    \bottomrule
    \end{tabular}
    \caption{Performance of LAC as base score on ImageNet LT ($\rho=0.6$) at target coverage $95\%$.}  
    \label{ImageNetClasswise_inefficency}
\end{table}

\end{document}